%% file: main.tex
\documentclass{article}

    \PassOptionsToPackage{numbers, compress}{natbib}

\usepackage[preprint]{neurips_2026}

\usepackage{amsmath}

\input{preamble}
\usepackage{eccvabbrv}
\usepackage{amsthm}
\newtheorem{proposition}{Proposition}
\newtheorem*{remark}{Remark}

\newif\ifarxiv
\newif\ifunified %
\arxivtrue %
\unifiedtrue

\ifunified
  \newcommand{\mcref}[1]{\cref{#1}}
  \newcommand{\mcite}[1]{\cite{#1}}
  \newcommand{\suppcref}[1]{\cref{#1}}
\else
  \newcommand{\suppcref}[1]{Supp.~\cref{#1}}
  \newcommand{\mcite}[1]{{\NoHyper\cite{#1}\endNoHyper}}
  \newcommand{\mcref}[1]{{\NoHyper\cref{#1}\endNoHyper}}
  \fi
  
\newcommand{\myparagraph}[1]{\noindent\textbf{#1}\enspace}

\usepackage[utf8]{inputenc} %
\usepackage[T1]{fontenc}    %
\usepackage{hyperref}       %
\usepackage{url}            %
\usepackage{booktabs}       %
\usepackage{amsfonts}       %
\usepackage{nicefrac}       %
\usepackage{microtype}      %
\usepackage{xcolor}         %

\hypersetup{
  colorlinks=true,
  linkcolor=blue,
  citecolor=blue,
  urlcolor=blue
}

\definecolor{ForestGreen}{RGB}{34,139,34}
\definecolor{BrickRed}{RGB}{178,34,34}

\usepackage{cleveref}

\AtEndPreamble{%
  \crefname{proposition}{Prop.}{Props.}%
  \Crefname{proposition}{Prop.}{Props.}%
  \crefname{algorithm}{Alg.}{Algs.}%
  \Crefname{algorithm}{Alg.}{Algs.}%
  \crefname{figure}{Fig.}{Figs.}%
  \Crefname{figure}{Fig.}{Figs.}%
  \crefname{table}{Tab.}{Tabs.}%
  \Crefname{table}{Tab.}{Tabs.}%
  \crefname{section}{Sec.}{Secs.}%
  \Crefname{section}{Sec.}{Secs.}%
  \crefname{subsection}{Sec.}{Secs.}%
  \Crefname{subsection}{Sec.}{Secs.}%
  \crefname{subsubsection}{Sec.}{Secs.}%
  \Crefname{subsubsection}{Sec.}{Secs.}%
  \crefname{equation}{Eq.}{Eqs.}%
  \Crefname{equation}{Eq.}{Eqs.}%
  \crefname{appendix}{App.}{Apps.}%
  \Crefname{appendix}{App.}{Apps.}%
}

\title{Designing Optimal Transport Flows}

\author{%
  \normalfont
  Shimon Malnick$^1$ \quad
  Matan Rusanovsky$^1$ \quad
  Ohad Fried$^2$ \quad
  Shai Avidan$^1$ \\[2pt]
  $^1$Tel Aviv University \quad $^2$Reichman University \\
  \texttt{\{malnick,matanru\}@mail.tau.ac.il} \quad
  \texttt{ofried@runi.ac.il} \quad
  \texttt{avidan@eng.tau.ac.il} \\
  Project page:~\href{https://www.malnick.net/designing_ot_flows}{malnick.net/designing\_ot\_flows}%
}

\begin{document}

\maketitle

\input{macros}

\input{sec/0_abstract}
\input{sec/1_intro}
\input{sec/2_related}

\input{sec/3_background}
\input{sec/4_method}
\input{sec/5_experiments}

\input{sec/6_limitations}

\input{sec/7_1_acks}

\bibliographystyle{plainnat}
\bibliography{main}

\input{sec/8_sup.tex}

\end{document}

%% file: preamble.tex
\newcommand{\cmark}{\ding{51}}%
\newcommand{\xmark}{\ding{55}}%

\usepackage{fix-cm}       %
\usepackage{subcaption}
\usepackage{caption}
\usepackage{adjustbox}
\usepackage{dcolumn}
\usepackage{xspace}
\usepackage{multirow}
\usepackage{makecell}
\usepackage{tikz}
\usetikzlibrary{arrows.meta,backgrounds,calc,fit, positioning, shapes.geometric, decorations.pathmorphing, external}
\usepackage{algorithm}
\usepackage{algpseudocode}
\usepackage{xcolor}
\usepackage{colortbl}
\usepackage{pifont}
\usepackage[normalem]{ulem}   %
\usepackage{enumitem}
\usepackage{wrapfig}        %
\usepackage{listings}

%% file: macros.tex
\newcommand{\uu}[1][t]{u^\theta_{#1}}
\newcommand{\Ours}{\textsc{Ours}\xspace}

\newcommand{\betweencellpdf}{\ensuremath{h^c}}
\newcommand{\betweencellpdffine}{\ensuremath{h^f}}
\newcommand{\approxbetweencellpdffine}{\ensuremath{\hat{h}^f}}
\newcommand{\incellpdf}{\ensuremath{f}}
\newcommand{\incellpdfnormalized}{\ensuremath{f'}}

\newcommand{\incellcdf}{\ensuremath{F}}

\newcommand{\totalpdf}{\ensuremath{f_{dd}}}
\newcommand{\totalcdf}{\ensuremath{F_{dd}}}

\newcommand{\ignorethis}[1]{}
\newcommand{\redund}[1]{#1}

\newcommand{\aposteriori}     {\textit{a~posteriori}}
\newcommand{\perse      }     {\textit{per~se}}
\newcommand{\naive      }     {{na\"{\i}ve}}
\newcommand{\Naive      }     {{Na\"{\i}ve}}
\newcommand{\Identity   }     {\mat{I}}
\newcommand{\Zero       }     {\mathbf{0}}
\newcommand{\Reals      }     {{\textrm{I\kern-0.18em R}}}
\newcommand{\isdefined  }     {\mbox{\hspace{0.5ex}:=\hspace{0.5ex}}}
\newcommand{\texthalf   }     {\ensuremath{\textstyle\frac{1}{2}}}
\newcommand{\half       }     {\ensuremath{\frac{1}{2}}}
\newcommand{\third      }     {\ensuremath{\frac{1}{3}}}
\newcommand{\fourth     }     {\ensuremath{\frac{1}{4}}}

\newcommand{\Lone} {\ensuremath{L_1}}
\newcommand{\Ltwo} {\ensuremath{L_2}}

\newcommand{\degree} {\ensuremath{^{\circ}}}

\newcommand{\mat        } [1] {{\text{\boldmath $\mathbit{#1}$}}}
\newcommand{\Approx     } [1] {\widetilde{#1}}
\newcommand{\change     } [1] {\mbox{{\footnotesize $\Delta$} \kern-3pt}#1}

\newcommand{\Order      } [1] {O(#1)}
\newcommand{\set        } [1] {{\lbrace #1 \rbrace}}
\newcommand{\floor      } [1] {{\lfloor #1 \rfloor}}
\newcommand{\ceil       } [1] {{\lceil  #1 \rceil }}
\newcommand{\inverse    } [1] {{#1}^{-1}}
\newcommand{\transpose  } [1] {{#1}^\mathrm{T}}
\newcommand{\invtransp  } [1] {{#1}^{-\mathrm{T}}}
\newcommand{\relu       } [1] {{\lbrack #1 \rbrack_+}}

\newcommand{\abs        } [1] {{| #1 |}}
\newcommand{\Abs        } [1] {{\left| #1 \right|}}
\newcommand{\norm       } [1] {{\| #1 \|}}
\newcommand{\Norm       } [1] {{\left\| #1 \right\|}}
\newcommand{\pnorm      } [2] {\norm{#1}_{#2}}
\newcommand{\Pnorm      } [2] {\Norm{#1}_{#2}}
\newcommand{\inner      } [2] {{\langle {#1} \, | \, {#2} \rangle}}
\newcommand{\Inner      } [2] {{\left\langle \begin{array}{@{}c|c@{}}
                               \displaystyle {#1} & \displaystyle {#2}
                               \end{array} \right\rangle}}

\newcommand{\twopartdef}[4]
{
  \left\{
  \begin{array}{ll}
    #1 & \mbox{if } #2 \\
    #3 & \mbox{if } #4
  \end{array}
  \right.
}

\newcommand{\fourpartdef}[8]
{
  \left\{
  \begin{array}{ll}
    #1 & \mbox{if } #2 \\
    #3 & \mbox{if } #4 \\
    #5 & \mbox{if } #6 \\
    #7 & \mbox{if } #8
  \end{array}
  \right.
}

\newcommand{\len}[1]{\text{len}(#1)}

\newlength{\w}
\newlength{\h}
\newlength{\x}

\definecolor{darkred}{rgb}{0.7,0.1,0.1}
\definecolor{darkgreen}{rgb}{0.1,0.6,0.1}
\definecolor{cyan}{rgb}{0.7,0.0,0.7}
\definecolor{otherblue}{rgb}{0.1,0.4,0.8}
\definecolor{maroon}{rgb}{0.76,.13,.28}
\definecolor{burntorange}{rgb}{0.81,.33,0}
\definecolor{cyan_real}{rgb}{0.0,1.0,1.0}
\definecolor{purple_real}{rgb}{0.5,0.0,0.5}
\definecolor{pixelblue}{HTML}{2E6FAC}
\definecolor{latentred}{HTML}{D95030}
\definecolor{traincolor}{rgb}{0.0,0.5,0.5}        %
\definecolor{infercolor}{rgb}{0.85,0.3,0.15}       %

\ifdefined\ShowNotes
  \newcommand{\colornote}[3]{{\color{#1}\textbf{#2} #3\normalfont}}
\else
  \newcommand{\colornote}[3]{}
\fi

\newcommand {\todo}[1]{\colornote{cyan}{TODO}{#1}}
\newcommand {\ohad}[1]{\colornote{burntorange}{OF:}{#1}}
\newcommand {\shai}[1]{\colornote{darkgreen}{SA:}{#1}}
\newcommand {\shimon}[1]{\colornote{otherblue}{SM:}{#1}}
\newcommand {\matan}[1]{\colornote{maroon}{MR:}{#1}}

\newcommand {\reqs}[1]{\colornote{red}{\tiny #1}}

\newcommand {\new}[1]{\colornote{red}{#1}}

\newcommand*\rot[1]{\rotatebox{90}{#1}}

\newcommand {\newstuff}[1]{#1}

\newcommand\todosilent[1]{}

\newcommand{\woBGmask}{{w/o~bg~\&~mask}}
\newcommand{\woMask}{{w/o~mask}}

\providecommand{\keywords}[1]
{
  \textbf{\textit{Keywords---}} #1
}

\newcommand{\R}{\mathbb{R}}
\newcommand{\D}{\mathcal{D}}
\newcommand{\N}{\mathcal{N}}
\newcommand{\DC}{\D_\mathcal{C}}
\newcommand{\DT}{\D_\mathcal{T}}
\newcommand{\X}{\mathcal{X}}
\newcommand{\Z}{\mathcal{Z}}

\newcommand{\probP}{\text{I\kern-0.15em P}}

\newcommand{\argmin}{\mathop{\mathrm{argmin}}}
\newcommand{\argmax}{\mathop{\mathrm{argmax}}}
\newcommand{\pr}[1]{\left(#1\right)}

\ifdefined\RevisedVersion
  \newcommand{\revision}[1]{{\textcolor{red}{#1}}}
  \long\def\trackchange#1#2{{\sout{#1}}\,{\color{red}#2}}
\else
  \newcommand{\revision}[1]{#1}
  \newcommand{\trackchange}[2]{#2}
\fi

\newcommand*{\nolink}[1]{%
  \begin{NoHyper}#1\end{NoHyper}%
}

\renewcommand{\cmark}{{\textcolor{ForestGreen}{\ding{51}}}}
\renewcommand{\xmark}{{\textcolor{BrickRed}{\ding{55}}}}

%% file: sec/0_abstract.tex
\begin{abstract}
Flow matching models learn to transport samples
from a simple prior distribution to a complex
data distribution.
When prior-data pairs are coupled via optimal transport (OT),
the learned trajectories are straight and non-crossing,
enabling fast, even single-step, generation.
However, computing the OT coupling
in high dimensions is intractable,
and existing methods attempt to solve the OT problem,
at the cost of persistent bias or significant overhead.
Rather than solving for the OT coupling, 
we reformulate the problem.
Once the prior is treated as a design choice rather than a fixed input,
the OT coupling between prior and data is no longer unique.
Many priors admit an OT-optimal identity coupling to the data,
leaving us free to choose one
that is also tractable to sample.
We identify low-frequency projection of natural images as such a choice.
The identity coupling between data
and its low-frequency representation
is empirically OT-optimal,
the prior is structured enough to be sampled
by a lightweight model at inference,
and the remaining flow-matching task reduces to
synthesizing high-frequency detail.
Interpolating the prior with Gaussian noise
further improves generation quality
while preserving the OT coupling.
The approach requires no modifications to the flow model itself,
and integrates naturally
with latent-space models,
classifier-free guidance,
and one-step generation frameworks.
Across all benchmarks,
our method reduces trajectory curvature
by more than $2\times$
compared to existing flow matching methods,
yielding better generation quality
in the few-step regime.
\end{abstract}

%% file: sec/1_intro.tex
\section{Introduction}
\label{sec:intro}

Generative models have become one of the most
transformative tools in modern machine learning.
In just a few years, they have scaled to produce
photorealistic images at high
resolution~\cite{esser2024scaling,flux2024,ma2024sit,rombach2022stable_diffusion},
generate temporally coherent
video~\cite{jin2025pyramidal,hacohen2024ltx,polyak2024movie,ho2022video},
and are rapidly expanding to new modalities
including audio~\cite{NEURIPS2023_2d8911db,mehta2024matcha,borsos2023audiolm}, 3D~\cite{hui2025notsooptimal,lan2025gaussiananything},
and beyond~\cite{jing2023alphafold,song2023equivariant}.
They are the backbone of
both open-source and commercial generation
systems~\cite{esser2024scaling,flux2024,polyak2024movie},
powering creative tools used by millions.

At the heart of every generative model lies a
single conceptual task,
learn to transform a simple, easy-to-sample distribution
into a complex, unknown target distribution from which we only observe finite samples.
If one could compute the optimal transport (OT) map
between a simple distribution and the true underlying data distribution,
generation would reduce to a single map evaluation,
transforming any noise sample directly into a data sample.
This would be the perfect generator,
a direct and cost-minimizing transformation
from noise to data,
and the ideal we aim to approach,
high-quality generation in as few steps as possible,
ideally just one.

But this map is out of reach.
Computing the optimal transport coupling
between a high-dimensional data distribution
and a standard Gaussian prior is computationally 
prohibitive at modern scales.
Exact solvers scale cubically in the number of
samples~\cite{kuhn1955hungarian},
while entropic relaxations~\cite{NIPS2013_af21d0c9}
remain prohibitive in the dimensionalities
of modern data.
Faced with this intractability,
the field turned to flow matching~\cite{lipman2022flow,liu2022flow,albergo2022building}.
Rather than computing the transport map directly,
flow matching learns a velocity field that
continuously deforms the simple distribution
into the data distribution over a unit time interval.
In standard practice, this velocity field is trained
on independently sampled noise-data pairs,
forgoing any global coupling structure.
This has proven remarkably effective,
and flow-based architectures now underpin
the leading generation
systems~\cite{esser2024scaling,flux2024,hacohen2024ltx,jin2025pyramidal}.

Yet flow matching inherits a fundamental tension
from the intractability it circumvents.
Since the true OT coupling is unknown,
standard training pairs noise and data 
independently at random.
This forces the learned velocity field to average
over conflicting transport directions,
producing curved, crossing trajectories
that demand many integration steps
at inference~\cite{pooladian2023multisample,tong2024improving}.
This is the central obstacle
to efficient few-step and single-step generation,
since straighter trajectories
require fewer steps to traverse.

\input{figs/teaser/fig_gaussians}
A number of works have sought to recover
what random pairing leaves on the table.
Minibatch OT
methods~\cite{pooladian2023multisample,tong2024improving}
solve small transport problems within each
training batch,
but these local approximations introduce
persistent bias that does not vanish with
training~\cite{kornilov2024optimal}.
To move beyond single-batch scope,
some methods accumulate and refine couplings
across minibatches throughout
training~\cite{davtyan2025faster},
while semi-discrete OT
approaches~\cite{kong2025alignflow,mousavi2025flow}
solve a global transport problem over the
full dataset as a preprocessing step,
adding significant overhead.
Distillation~\cite{liu2022flow,seongbalanced}
and trajectory
optimization~\cite{yue2025oat}
straighten paths post-training,
adding further stages on top of a pretrained model.
All of these methods treat the OT problem
as something to be solved.

We argue it can be avoided by design,
treating the prior itself as a variable to be chosen
rather than a fixed input.
As illustrated in \cref{fig:teaser_gaussian_a},
standard flow matching pairs a Gaussian distribution
with a complex target through random coupling,
producing crossing trajectories
that curve the learned velocity field and
demand many integration steps at inference.
Rather than attempting to solve or approximate
the OT coupling for a fixed Gaussian prior,
we ask a different question.
Can we \emph{choose} a prior distribution
for which optimal transport is achieved by design?
Concretely, we seek a \emph{new prior distribution},
defined by a transformation of the data,
such that the induced coupling is OT-optimal,
the new prior can be sampled efficiently
and generation from the target distribution remains faithful.
\cref{fig:teaser_gaussian_b} illustrates this idea,
where a structured intermediate distribution $\tilde{p}_0$
replaces the Gaussian.
Its coupling to the data is OT-optimal by design,
while the remaining mapping from Gaussian noise to $\tilde{p}_0$
only needs to generate coarse image structure
in a lower-dimensional space.

Such a prior can be obtained using a simple, fixed transformation
that requires no data-dependent learning or optimization.
Leveraging the spectral concentration
of natural images~\cite{ruderman1994statistics},
we project each image onto its low-frequency representation.
Most of the $\ell_2$ energy of a natural image
lives in its low frequencies,
so the projection retains coarse structure
while discarding only fine detail.
This yields a prior that is informative enough
to keep distinct images well separated,
and simple enough to be sampled by a lightweight generator.
The identity coupling between an image
and its low-frequency representation
is empirically OT-optimal,
recovering an OT pair without solving any OT problem.

This decomposes generation
into a low-dimensional problem
and an OT-structured one.
The first stage operates
in a much smaller space,
generating only coarse image structure
at a fraction of the original dimensionality.
The second stage synthesizes
the remaining high-frequency details
at full resolution,
precisely where the identity coupling
between projection and data is OT-optimal by design,
requiring no solver or precomputation.
A deterministic projection prior
lies on a low-dimensional subspace,
a known difficulty for velocity-field learning~\cite{song2019generative}.
We interpolate the projection with Gaussian noise,
empirically preserving the OT-identity coupling
while improving generation quality.
 
 Our contributions are as follows:
\textbf{(1)} we propose \emph{OT by design},
a new perspective on optimal transport in flow matching.
Rather than solving an OT problem for a fixed prior,
we design the prior so that its identity coupling to the data
is empirically OT-optimal,
a property that other methods spend significant effort
to approximate.
 \textbf{(2)} We construct this prior using low-frequency projections of training images,
verify the resulting identity coupling is OT-optimal with the Hungarian algorithm,
add a noise interpolation step that improves generation quality
while preserving the coupling,
and show that OT-optimality alone
is not sufficient without low-frequency structure.
 \textbf{(3)} We show that this design consistently reduces trajectory curvature by more than $2\times$ over existing flow matching methods,
  yields strong few-step and single-step generation
across CIFAR-10, FFHQ, and ImageNet, and integrates naturally with latent-space models, classifier-free guidance,
   and one-step frameworks such as MeanFlow~\cite{geng2025mean}, 
   where it surpasses both the 1-step and 2-step baselines in our experiments.

%% file: figs/teaser/fig_gaussians.tex
\providecommand{\gfigdistsize}{\normalsize}   %
\providecommand{\gfiglabelsize}{\Large}   %
\providecommand{\gfigbotsize}{\normalsize}    %
\definecolor{srccolor}{RGB}{55,90,200}
\definecolor{tgtcolor}{RGB}{200,55,55}
\definecolor{projcolor}{RGB}{30,140,60}
\definecolor{arrowgray}{RGB}{140,140,140}
\definecolor{labelRedG}{RGB}{185,35,35}
\definecolor{labelGreenG}{RGB}{22,138,55}

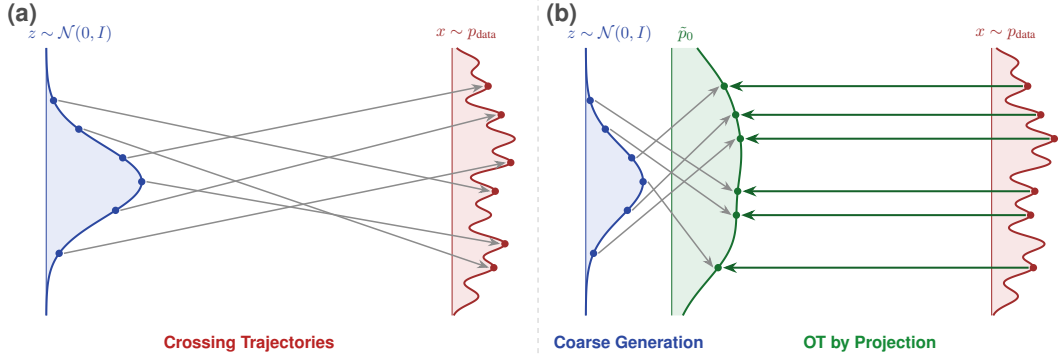
\begin{figure}[ht]
    \centering
    \resizebox{\textwidth}{!}{%
    \begin{tikzpicture}[
        >=Stealth,
        declare function={
            gauss(\x,\m,\s,\a) = \a*exp(-(\x-\m)^2/(2*\s^2));
            hfmog(\x) = gauss(\x,-2.4,0.17,0.55) + gauss(\x,-1.8,0.16,0.85)
                + gauss(\x,-1.3,0.18,1.10) + gauss(\x,-0.7,0.15,0.80)
                + gauss(\x,-0.2,0.17,0.90) + gauss(\x,0.4,0.16,1.20)
                + gauss(\x,0.9,0.18,1.30) + gauss(\x,1.4,0.15,1.00)
                + gauss(\x,2.0,0.17,0.75) + gauss(\x,2.5,0.16,0.55);
            lfenv(\x) = gauss(\x,-1.3,0.85,1.10) + gauss(\x,0.5,0.90,1.20)
                + gauss(\x,2.0,0.80,0.80);
        }
    ]

    \pgfmathsetmacro{\rtx}{8.5}                              %
    \pgfmathsetmacro{\rtxBpad}{0.0}                          %
    \pgfmathsetmacro{\rtxB}{\rtx + \rtxBpad}                 %
    \pgfmathsetmacro{\gap}{1.0}                              %
    \pgfmathsetmacro{\txA}{\rtx}                             %
    \pgfmathsetmacro{\divx}{\rtx + 0.8 + \gap}              %
    \pgfmathsetmacro{\bshift}{\divx + \gap}                  %

    \begin{scope}[shift={(0,0)}]

        \node[font=\gfiglabelsize\sffamily\bfseries, text=black!70]
            at (-0.5, 3.5) {(a)\phantomsubcaption\label{fig:teaser_gaussian_a}};

        \fill[srccolor!12]
            (0,-2.8) -- plot[domain=-2.8:2.8, samples=100, smooth]
            ({gauss(\x,0,0.75,2.0)}, \x) -- (0,2.8) -- cycle;
        \draw[srccolor!80!black, line width=1.2pt]
            plot[domain=-2.8:2.8, samples=100, smooth]
            ({gauss(\x,0,0.75,2.0)}, \x);
        \draw[srccolor!80!black, line width=0.6pt] (0,-2.8) -- (0,2.8);
        \node[srccolor!80!black, font=\gfigdistsize] at (0.5, 3.1)
            {$z \sim \mathcal{N}(0, I)$};

        \fill[tgtcolor!12]
            (\txA,-2.8) -- plot[domain=-2.8:2.8, samples=200, smooth]
            ({\txA + hfmog(\x)}, \x)
            -- (\txA,2.8) -- cycle;
        \draw[tgtcolor!80!black, line width=1.2pt]
            plot[domain=-2.8:2.8, samples=200, smooth]
            ({\txA + hfmog(\x)}, \x);
        \draw[tgtcolor!80!black, line width=0.6pt] (\txA,-2.8) -- (\txA,2.8);
        \node[tgtcolor!80!black, font=\gfigdistsize] at (\txA+0.3, 3.1)
            {$x \sim p_{\text{data}}$};

        \foreach \ys/\yt in {
            -1.5/ 0.4,
            -0.6/ 1.4,
             0.0/-1.3,
             0.5/ 2.0,
             1.1/-1.8,
             1.7/-0.2%
        } {
            \pgfmathsetmacro{\sx}{gauss(\ys,0,0.75,2.0)}
            \pgfmathsetmacro{\ex}{\txA + hfmog(\yt)}
            \fill[srccolor!80!black] (\sx, \ys) circle (2.2pt);
            \fill[tgtcolor!80!black] (\ex, \yt) circle (2.2pt);
            \draw[->, arrowgray, line width=0.9pt]
                ({\sx+0.1}, \ys) -- ({\ex-0.1}, \yt);
        }

        \node[font=\gfigbotsize\sffamily\bfseries, text=labelRedG]
            at ({\txA/2}, -3.4) {\strut Crossing Trajectories};

    \end{scope}

    \begin{scope}[shift={(\bshift,0)}]

        \node[font=\gfiglabelsize\sffamily\bfseries, text=black!70]
            at (-0.5, 3.5) {(b)\phantomsubcaption\label{fig:teaser_gaussian_b}};

        \fill[srccolor!12]
            (0,-2.8) -- plot[domain=-2.8:2.8, samples=100, smooth]
            ({gauss(\x,0,0.75,1.2)}, \x) -- (0,2.8) -- cycle;
        \draw[srccolor!80!black, line width=1.2pt]
            plot[domain=-2.8:2.8, samples=100, smooth]
            ({gauss(\x,0,0.75,1.2)}, \x);
        \draw[srccolor!80!black, line width=0.6pt] (0,-2.8) -- (0,2.8);
        \node[srccolor!80!black, font=\gfigdistsize] at (0.5, 3.1)
            {$z \sim \mathcal{N}(0, I)$};

        \pgfmathsetmacro{\px}{1.8}
        \fill[projcolor!12]
            (\px,-2.8) -- plot[domain=-2.8:2.8, samples=100, smooth]
            ({\px + lfenv(\x)}, \x) -- (\px,2.8) -- cycle;
        \draw[projcolor!80!black, line width=1.2pt]
            plot[domain=-2.8:2.8, samples=100, smooth]
            ({\px + lfenv(\x)}, \x);
        \draw[projcolor!80!black, line width=0.6pt] (\px,-2.8) -- (\px,2.8);
       \node[projcolor!80!black, font=\gfigdistsize, align=center]
            at ({\px+0.3}, 3.1)
            {$\tilde{p}_0$};

        \fill[tgtcolor!12]
            (\rtxB,-2.8) -- plot[domain=-2.8:2.8, samples=200, smooth]
            ({\rtxB + hfmog(\x)}, \x)
            -- (\rtxB,2.8) -- cycle;
        \draw[tgtcolor!80!black, line width=1.2pt]
            plot[domain=-2.8:2.8, samples=200, smooth]
            ({\rtxB + hfmog(\x)}, \x);
        \draw[tgtcolor!80!black, line width=0.6pt] (\rtxB,-2.8) -- (\rtxB,2.8);
        \node[tgtcolor!80!black, font=\gfigdistsize] at (\rtxB+0.3, 3.1)
            {$x \sim p_{\text{data}}$};

        \foreach \ys/\yp in {
            -1.5/ 0.9,
            -0.6/ 1.4,
             0.0/-1.8,
             0.5/ 2.0,
             1.1/-0.7,
             1.7/-0.2%
        } {
            \pgfmathsetmacro{\sx}{gauss(\ys,0,0.75,1.2)}
            \pgfmathsetmacro{\gx}{\px + lfenv(\yp)}
            \fill[srccolor!80!black] (\sx, \ys) circle (2.2pt);
            \fill[projcolor!80!black] (\gx, \yp) circle (2.2pt);
            \draw[->, arrowgray, line width=0.9pt]
                ({\sx+0.1}, \ys) -- ({\gx-0.1}, \yp);
        }

        \foreach \yt/\yp in {
            -1.8/-1.8,
            -0.7/-0.7,
            -0.2/-0.2,
             0.9/ 0.9,
             1.4/ 1.4,
             2.0/ 2.0%
        } {
            \pgfmathsetmacro{\ex}{\rtxB + hfmog(\yt)}
            \pgfmathsetmacro{\gx}{\px + lfenv(\yp)}
            \fill[tgtcolor!80!black] (\ex, \yt) circle (2.2pt);
            \draw[->, projcolor!70!black, line width=1.2pt]
                ({\ex-0.1}, \yt) -- ({\gx+0.1}, \yp);
        }
        
        \pgfmathsetmacro{\otmid}{(\px + \rtxB) / 2 + 0.8}  %
        \node[font=\gfigbotsize\sffamily\bfseries, text=labelGreenG]
            at (\otmid, -3.4) {\strut OT by Projection};
        \node[font=\gfigbotsize\sffamily\bfseries, text=srccolor!80!black]
            at ({\px/2}, -3.4) {\strut Coarse Generation};

    \end{scope}

    \draw[black!20, line width=0.5pt, dashed] (\divx, -3.8) -- (\divx, 4.0);

    \end{tikzpicture}%
    }%
\caption{
\textbf{Why Solve OT When You Can Design It?}
\textbf{(a)} Random noise-data pairing in standard flow matching~\cite{lipman2022flow} induces crossing trajectories, 
motivating an OT coupling.
\textbf{(b)} Projecting samples onto their low-frequency representation
 defines an intermediate distribution $\tilde{p}_0$ whose coupling to the data is OT-optimal by design,
  while the mapping from Gaussian noise to $\tilde{p}_0$ reduces to coarse generation.
}
    \label{fig:teaser_gaussian}
\end{figure}

%% file: sec/2_related.tex
\section{Related work}
\label{sec:related}

Flow matching~\cite{lipman2022flow,albergo2022building,liu2022flow}
trains continuous normalizing flows~\cite{chen2018neural},
enabling simulation-free optimization on linear interpolation paths,
and now powers state-of-the-art image
synthesis~\cite{flux2024,esser2024scaling}.
Pairing noise and data independently during training
produces curved generation
trajectories~\cite{tong2024improving,pooladian2023multisample}
that demand many Ordinary Differential Equation (ODE) solver steps at inference.
Two complementary lines of work address this,
improving the noise-data coupling (\cref{sec:related_straight}),
and replacing  the pure-noise prior
with more informative starting distributions (\cref{sec:rw_priors}).
Our work connects them,
showing that a simple choice of prior
can simultaneously yield OT couplings
and straight trajectories.

\subsection{Straightening generation trajectories}
\label{sec:related_straight}

\myparagraph{Minibatch OT coupling.}
These methods~\cite{pooladian2023multisample,tong2024improving}
solve an OT problem within each training batch.
The resulting coupling is locally straighter
but biased~\cite{kornilov2024optimal,mousavi2025fitting},
and per-batch assignment scales poorly with batch size and dimension~\cite{zhang2025fitting,mousavi2025fitting}.
A recent variant~\cite{lin2025optimaltransportmodelalignedcoupling}
selects pairs by model prediction error rather than geometric cost,
but remains within single batches.

\myparagraph{Global OT coupling.}
Several methods compute a coupling across the full dataset.
One propagates assignments across overlapping batches during training~\cite{davtyan2025faster},
while AlignFlow~\cite{kong2025alignflow} and SD-FM~\cite{mousavi2025flow}
solve a semi-discrete OT problem offline.
These yield higher-quality couplings than minibatch OT,
but require dataset-wide precomputation
and tie assignments to specific training samples
rather than learning a source distribution.

\myparagraph{Post-training straightening.}
A separate line of work straightens trajectories
after the base model has been trained,
through iterative self-distillation~\cite{liu2022flow},
guided distillation~\cite{seongbalanced},
or action minimization~\cite{yue2025oat}.
These methods are complementary to the choice of coupling
and can be applied on top of any training-time method,
including ours,
to further refine trajectories.

\myparagraph{Few-step generation via modified objectives.}
Consistency models~\cite{pmlr-v202-song23a,yang2024consistency},
shortcut models~\cite{frans2025one},
and MeanFlow~\cite{geng2025mean}
modify the training objective itself for few-step or single-step generation.
Because they leave the noise-data coupling unchanged,
they are complementary to our prior design and can be composed with it,
as we demonstrate by integrating our prior into MeanFlow in \cref{sec:latent_results}.

\input{tables/related}
\subsection{Alternative priors for generation}
\label{sec:rw_priors}
The standard isotropic Gaussian prior discards all data structure.
A growing body of work replaces or augments this pure-noise prior with degradations that retain partial information about the target.
Cold Diffusion~\cite{bansal2023cold} showed that purely deterministic degradations
such as blur, downsampling, and masking
suffice for generation when inverted by a learned model.
Subsequent work formalized frequency-aware destruction~\cite{hoogeboom2023blurring},
interpolated between blur and noise~\cite{daras2023soft},
and corrupted Fourier magnitude while preserving spatial phase~\cite{zeng2025neuralremaster}.
Anisotropic Gaussians capturing local manifold geometry~\cite{bamberger2025carr}
and time-varying blue noise~\cite{huang2024blue}
have similarly been shown to improve flow matching and diffusion respectively.
In the conditional setting,
priors reflecting class or prompt structure shorten transport paths~\cite{issachar2025designing, kim2026better}.

Our work shares this motivation,
but rather than designing noise distributions or degradation processes,
we show that the choice of prior can be made so that its identity coupling
to the data is empirically OT-optimal,
yielding an informative starting point
and straighter trajectories
without solving or approximating an OT problem.
\cref{tab:related_comparison} summarizes the properties
of these methods, showing that our method operates within the standard flow matching framework
without requiring any OT solver.

%% file: tables/related.tex
\begin{wraptable}[14]{R}{0.6\linewidth}
\centering
\vspace{-0.4cm}
\caption{\textbf{Trade-Offs in Trajectory Straightening.}
\textit{From Scratch}: no pre-trained model required.
\textit{OT-Based}: coupling derived from optimal transport.
\textit{Scales w/ $B$}: no per-batch OT solve.
\textit{Scales w/ $N$}: no dataset-wide precomputation.
\textit{Dataset-agnostic}: no assignment tied to training samples.
 }
\resizebox{1.0\linewidth}{!}{%
\begin{tabular}{@{}lccccc@{}}
\toprule
 & \makecell{From\\Scratch}
 & \makecell{OT-\\Based}
 & \makecell{Scales w/\\batch $B$}
 & \makecell{Scales w/\\data $N$}
 & \makecell{Dataset-\\agnostic} \tabularnewline
\midrule
OT-FM / Multisample FM~\cite{tong2024improving,pooladian2023multisample}
  & \cmark & \cmark & \xmark & \cmark & \cmark \tabularnewline
OFM~\cite{kornilov2024optimal}
  & \cmark & \cmark & \xmark & \xmark & \xmark \tabularnewline
LOOM-CFM~\cite{davtyan2025faster}
  & \cmark & \cmark & \xmark & \cmark & \xmark \tabularnewline
AlignFlow~\cite{kong2025alignflow}
  & \cmark & \cmark & \cmark & \xmark & \xmark \tabularnewline
SD-FM~\cite{mousavi2025flow}
  & \cmark & \cmark & \cmark & \xmark & \xmark \tabularnewline
Lee \etal~\cite{lee2023minimizing}
  & \cmark & \xmark & \cmark & \cmark & \cmark \tabularnewline
Reflow~\cite{liu2022flow}
  & \xmark & \xmark & \cmark & \cmark & \cmark \tabularnewline
Conic Reflow~\cite{seongbalanced}
  & \xmark & \xmark & \cmark & \cmark & \cmark \tabularnewline
OAT-FM~\cite{yue2025oat}
  & \xmark & \cmark & \cmark & \cmark & \cmark \tabularnewline
\textbf{Ours}
  & \cmark & \cmark & \cmark & \cmark & \cmark \tabularnewline
\bottomrule
\end{tabular}%
}
\label{tab:related_comparison}
\end{wraptable}

%% file: sec/3_background.tex
\section{Background}
\label{sec:background}

\subsection{Flow matching}

Let $p_0$ and $p_1$ denote continuous distributions over $\mathbb{R}^d$,
serving as the base (noise) and data distributions, respectively.
Flow matching~\cite{lipman2022flow,albergo2022building,liu2022flow} learns a time-dependent flow $\phi_t: \mathbb{R}^d \rightarrow \mathbb{R}^d$ for $t \in [0, 1]$, 
which maps each sample $x_0 \sim p_0$ along a continuous path such that $\phi_0(x_0) = x_0$ and  $\phi_1(x_0) \sim p_1$,
\ie, $\phi_1$ pushes $p_0$ forward to $p_1$.
For brevity, we denote $x_t = \phi_t(x_0)$. 
The evolution of $x_t$ is governed by an Ordinary Differential Equation (ODE):
\begin{equation}
\label{eq:fm_ode}
    \frac{d}{dt}x_t = u_t(x_t),
\end{equation}
where $u_t: \mathbb{R}^d \rightarrow \mathbb{R}^d$, for $t \in [0,1]$,
is the velocity field that generates the flow.
The goal of flow matching is to learn a neural velocity field $v_\theta(x_t, t)$
that approximates $u_t(x_t)$.

Directly regressing onto $u_t$ is intractable,
but an equivalent \emph{conditional} objective
can be optimized instead~\cite{lipman2022flow}.
Under a linear interpolation path:
\begin{equation}
    \label{eq:fm_path}
    \phi_t(x_0 | x_1) = x_t = t x_1 + (1-t)x_0,
\end{equation}
where $x_0 \sim p_0, x_1 \sim p_1$ are sampled independently.
The corresponding conditional velocity is $u_t(x_t | x_1) = x_1 - x_0$,
and the training objective reduces to:
\begin{equation}
    \label{eq:fm_loss}
    \mathcal{L}(\theta) = \mathop{\mathbb{E}}\limits_{t,\, x_0,\, x_1}
    \left[\|v_\theta(x_t, t) - (x_1 - x_0)\|^2_2\right],
\end{equation}
where $t \sim \mathcal{U}[0, 1], x_0 \sim p_0, x_1 \sim p_1$.
At inference, new samples are generated by solving the ODE in \cref{eq:fm_ode}
from $t=0$ to $t=1$, initialized from $x_0 \sim p_0$.

While each conditional velocity $u_t(x_t | x_1) = x_1 - x_0$ defines a straight path
between a specific pair $(x_0, x_1)$,
the independent sampling of $x_0$ and $x_1$ in \cref{eq:fm_loss} means that
these paths may cross one another~(\cref{fig:teaser_gaussian_a})~\cite{geng2025mean,pooladian2023multisample}.
The learned velocity field $v_\theta$ must average over such conflicting directions,
resulting in curved trajectories at inference time.
Straightening these trajectories is desirable,
as it reduces the number of ODE solver steps required for generation.

\subsection{Optimal transport for straightening trajectories}
Given a map $T: \mathbb{R}^d \rightarrow \mathbb{R}^d$, 
we denote the push-forward of $p_1$ by $T$ as $T_\# (p_1)$, \ie, the distribution of $T(x)$ for $x \sim p_1$.
The Monge optimal transport problem
seeks a map that transports $p_1$ to $p_0$
with minimal squared Euclidean cost: 
\begin{equation}
    \label{eq:ot}
    \inf_{T:\, T_\#(p_1) = p_0}
    \mathbb{E}_{x_1 \sim p_1} \left[ \|T(x_1) - x_1\|^2 \right].
\end{equation}

In the context of flow matching, replacing the independent coupling
in \cref{eq:fm_loss} with the OT-optimal pairing eliminates path crossings,
since any crossing could be ``uncrossed''
to reduce the total cost, contradicting optimality.
The resulting velocity field produces
significantly straighter
trajectories~\cite{pooladian2023multisample,tong2024improving},
reducing the number of ODE steps needed at inference.

However, computing the OT coupling
between a high-dimensional data distribution
and a standard Gaussian prior is intractable.
While the optimal map is guaranteed to
exist~\cite{brenier1991polar},
obtaining it requires $O(n^3)$ time
and $O(n^2)$ memory in the number of
samples~\cite{kuhn1955hungarian,NIPS2013_af21d0c9},
rendering it infeasible at modern data scales.
As discussed in \cref{sec:intro,sec:related},
existing methods attempt to solve or approximate
this coupling at significant cost.
Next,
we show it can be  
sidestepped
by designing a prior whose identity coupling to the data
is OT-optimal by design.

%% file: sec/4_method.tex
\section{Method}
\label{sec:method}

Our approach replaces the standard Gaussian prior
with a structured distribution
obtained by projecting each training image
onto its low-frequency representation.
\cref{fig:pipeline} illustrates the full pipeline.
We first describe the prior design (\cref{sec:ot_by_design}),
then the training and inference procedure (\cref{sec:pipeline}).

\input{figs/pipeline/pipeline}
\subsection{Designing the OT prior}
\label{sec:ot_by_design}

Rather than solving the OT problem for a fixed prior,
we treat the prior itself as a design choice.
Given the data distribution $p_1$,
we seek a transformation $T: \mathbb{R}^d \to \mathbb{R}^d$
that defines the prior as the distribution it induces,
$\tilde{p}_0 = T_{\#}(p_1)$,
with two design goals.
(i) The identity coupling between data and prior
should be OT-optimal,
yielding straight, non-crossing trajectories
for the flow model.
Under the squared-cost OT objective of \cref{eq:ot},
the identity coupling pays
the mean squared reconstruction error of $T$,
so this goal favors $T$
with small reconstruction error on $p_1$.
We discuss strict orthogonal projections,
for which the identity coupling is exactly OT-optimal,
in \suppcref{sec:supp_projection_variants},
with \suppcref{rem:parseval}
connecting these to the downsampling operators
used in the pipeline.
(ii) The prior should be expressive enough
to support generation of $p_1$,
yet structured enough to be sampled
by a lightweight model at inference.
The two goals are not equivalent,
and (i) alone is not sufficient,
as we show below.

We argue for low-frequency projections of natural images.
The power spectrum of natural images
follows a $1/f^2$
decay~\cite{field1987relations,ruderman1994statistics},
so low-pass filtering retains most of each image's $\ell_2$ energy
and incurs only small reconstruction error,
addressing goal (i).
The resulting prior consists of low-resolution images
capturing coarse structure,
simple enough for a lightweight model to sample,
addressing (ii).

To check that the two goals are genuinely independent,
we compare four projection variants for $T$.
Fourier truncation, random pixel masking, and random patch masking
are strict orthogonal projections
and therefore provably preserve the OT-identity coupling
(\suppcref{prop:ot_discrete}).
Only the low-frequency variants
yield low FID and low trajectory curvature
under noise perturbation
(\suppcref{tab:projection_variants}).
OT-optimality of the identity coupling
is therefore not sufficient,
the prior must also retain low-frequency structure
that a lightweight model can sample efficiently.
Among the low-frequency variants,
bicubic downsampling yields a small but consistent FID improvement
over Fourier truncation,
and we adopt it.

The resulting design relates to a recurring principle
in generative modeling.
Cascaded diffusion~\cite{ho2022cascaded,saharia2022photorealistic},
latent diffusion~\cite{rombach2022stable_diffusion},
and pyramid flow~\cite{jin2025pyramidal}
all decompose generation across resolutions or feature scales,
generating coarse structure first
and refining it with high-frequency detail.
Our contribution is to make this decomposition explicit
at the level of the flow's coupling.
The flow model never has to invent coarse structure,
its starting point already contains it,
and the residual task of synthesizing high-frequency detail
is exactly the regime where the OT-identity coupling holds.

\input{figs/swaps_bars/swaps_bars}

The deterministic coupling defined by $T$
pairs each $x_1$ with a single fixed $T(x_1)$.
While this is the exact coupling we sought,
the resulting prior concentrates
on a low-dimensional subspace,
where nearby prior samples
can be matched to distant data points.
The velocity field is then forced
to average over conflicting transport directions
within small neighborhoods,
producing curved trajectories at inference time.
This echoes the observation in score-based
generative modeling~\cite{song2019generative}
that learning a stable vector field
over data concentrated on a low-dimensional manifold
is difficult.
Deterministic coupling also exposes the model
to the same pair $(T(x_1), x_1)$ every epoch,
risking overfitting.
Adding noise alleviates both issues
by spreading the prior
and breaking the pairing.
We formalize this by interpolating
the low-frequency representation and Gaussian noise:
\begin{equation}
    \label{eq:noisy_prior}
    x_0 = (1 - \alpha)\, T(x_1) + \alpha\, \epsilon,
    \quad \epsilon \sim \mathcal{N}(0, I),
\end{equation}
where $\alpha \in [0,1]$ controls the noise ratio.
$\alpha = 0$ recovers the deterministic coupling, 
and $\alpha = 1$ reduces to standard flow matching
with random pairings.

The expected squared distance
between two noisy prior samples decomposes as
\begin{equation}
    \label{eq:noisy_distance}
    \mathbb{E}\left[\|x_0^{(i)} - x_0^{(j)}\|^2\right]
    = (1 - \alpha)^2\, \mathbb{E}\left[\|T(x_1^{(i)}) - T(x_1^{(j)})\|^2\right]
    + \alpha^2\, \mathbb{E}\left[\|\epsilon_i - \epsilon_j\|^2\right].
\end{equation}
The first term is the mean pairwise distance
in the projected subspace,
while the second evaluates to $2\alpha^2 d$,
growing linearly with the data dimensionality,
so even moderate noise levels provide substantial separation
between prior samples (\suppcref{fig:ot:pairwise_dist_noise}).

A natural question is whether the OT-identity coupling
survives noise interpolation.
To test this, we compute the exact discrete OT assignment (Hungarian algorithm)
between $10{,}000$ noisy prior samples
and their corresponding data samples at different noise levels,
and measure the fraction of pairs
whose optimal assignment matches the identity coupling.
\cref{fig:ot:noise_vs_ot_swaps} shows
that the identity coupling remains OT-optimal 
up to $\alpha \approx 0.5$ (shown as $50\%$ noise ratio in the figure),
in both pixel space (CIFAR-10, $99.9\%$ preservation) and latent space (FFHQ, $100\%$), 
with the sample images above the bars
showing progression in $\alpha$.
Only beyond $\alpha = 0.5$ does noise disrupt the assignment structure.
\cref{sec:supp_pairwise}
extends this analysis with a finer $\alpha$ sweep
and the full derivation of \cref{eq:noisy_distance}.

Moderate noise preserves the OT coupling
and improves generation quality,
with FID minimized near $\alpha = 0.5$.
Quality improves steadily
from $\alpha = 0$ to $\alpha = 0.5$,
then degrades as noise overwhelms
the low-frequency structure.
At lower values, prior samples cluster tightly
in the projected subspace,
risking overfitting and trajectory crossings.
The $\alpha = 0$ endpoint
is itself a meaningful baseline,
a pure coarse-to-fine pipeline
in which $v_\theta$ refines a deterministic projection prior
without any noise interpolation.
The FID gap to $\alpha = 0.5$ in \suppcref{fig:ot:fid_vs_noise}
thus measures the contribution
of the OT-preserving noise interpolation
beyond a coarse-to-fine substrate alone.
We adopt $\alpha = 0.5$ as the operating point
for all our experiments
(full FID sweep in \suppcref{fig:ot:fid_vs_noise}).

\subsection{Pipeline and training}
\label{sec:pipeline}
We instantiate $T = \mathcal{U} \circ \mathcal{D}$,
where $\mathcal{D} \colon \mathbb{R}^d \to \mathbb{R}^{d'}$ produces
$x_1^{\downarrow} = \mathcal{D}(x_1) \in \mathbb{R}^{d'}$,
and $\mathcal{U}: \mathbb{R}^{d'} \to \mathbb{R}^d$
is its corresponding upsampling,
so that $T$ maps $x_1$
to its low-frequency approximation in $\mathbb{R}^d$.
During training (\cref{fig:pipeline_a}),
$T(x_1)$ is computed directly from each data sample
and perturbed via \cref{eq:noisy_prior}
to form the starting point $x_0$,
which is paired with $x_1$
for standard flow matching training of $v_\theta$ (see \suppcref{alg:coupling}).
At inference (\cref{fig:pipeline_c}),
a lightweight generator $G_\phi$
maps Gaussian noise to $\mathbb{R}^{d'}$.
In practice, $G_\phi$ is itself a flow matching model
operating at a small fraction of the original dimensionality
(architecture details in \cref{sec:experiments}),
trained separately (\cref{fig:pipeline_b}) to produce $\hat{x}_1^{\downarrow}$,
which is deterministically upsampled
to $\mathcal{U}(\hat{x}_1^{\downarrow})$
and perturbed with noise.
The flow model $v_\theta$, under the OT-identity coupling,
synthesizes the remaining high-frequency details,
with no OT problem to solve
and no architectural changes.

%% file: figs/pipeline/pipeline.tex
\begin{figure}[ht]
    \centering

    \def\boxpadtop{6.5pt}       %
    \def\boxpadbot{22.5pt}      %
    \def\rowgap{0.00cm}       %
    \def\imgw{1.0cm}          %
    \def\imgwsm{0.5cm}        %
    \def\bcgap{0.2cm}         %
    \def\apadtop{1pt}         %
    \def\apadright{2pt}       %

    \scalebox{0.91}{%
    \begin{tikzpicture}[
        node distance=0.5cm and 0.4cm,
        img/.style={inner sep=0pt},
        op/.style={draw, rounded corners, fill=gray!10, minimum height=0.6cm, minimum width=0.8cm, font=\scriptsize},
        arrow/.style={->, >=stealth, thick},
        label/.style={font=\scriptsize, text=black},
    ]

    \node[img] (highres) {\includegraphics[width=\imgw]{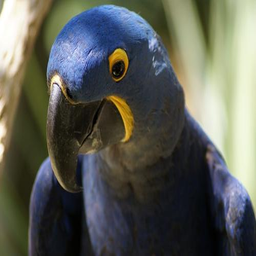}};
    \node[label, below=0.05cm of highres] {$x_1 \sim p_1$};

    \node[op, right=0.5cm of highres] (lowpass) {$\mathcal{D}$};

    \node[img, right=0.5cm of lowpass] (lowres) {\includegraphics[width=\imgwsm]{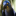}};
    \node[label, below=0.05cm of lowres] {$x_1^{\downarrow}$};

    \node[op, right=0.5cm of lowres] (upsample) {$\mathcal{U}$};

    \node[img, right=0.5cm of upsample] (low2high) {\includegraphics[width=\imgw]{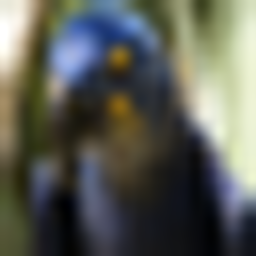}};
    \node[label, below=0.05cm of low2high] {$\mathcal{U}(x_1^{\downarrow})$};

    \node[op, right=0.5cm of low2high] (addnoise) {$\alpha$};

    \node[img, right=0.5cm of addnoise] (noisylow2high) {\includegraphics[width=\imgw]{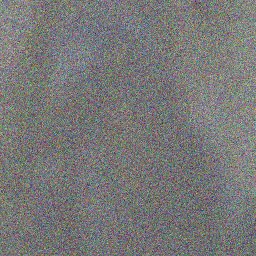}};
    \node[label, below=0.05cm of noisylow2high] {$x_0$};

    \node[op, right=0.55cm of noisylow2high, fill=blue!10, minimum width=1.5cm, minimum height=0.7cm] (fm) {$v_\theta(x_t, t)$};

    \node[img, right=0.55cm of fm] (output) {\includegraphics[width=\imgw]{figs/pipeline/training_imgs/high_res.png}};
    \node[label, below=0.05cm of output] {$x_1$};

    \draw[arrow] (highres) -- (lowpass);
    \draw[arrow] (lowpass) -- (lowres);
    \draw[arrow] (lowres) -- (upsample);
    \draw[arrow] (upsample) -- (low2high);
    \draw[arrow] (low2high) -- (addnoise);
    \draw[arrow] (addnoise) -- (noisylow2high);
    \draw[arrow] (noisylow2high) -- (fm);
    \draw[arrow] (fm) -- (output);

    \coordinate (fs) at ([shift={(-18.5pt,2pt)}]current bounding box.south west);
    \coordinate (fn) at ([shift={(\apadright,\apadtop)}]current bounding box.north east);
    \draw[dashed, traincolor!60, rounded corners=4pt] (fs) rectangle (fn);
    \node[font=\scriptsize\bfseries, text=traincolor!70, fill=white, inner sep=2pt, anchor=south west] at (fs) {(a)};

    \end{tikzpicture}%
    }%

    \vspace{\rowgap}

    \scalebox{0.91}{%
    \begin{tikzpicture}[
        node distance=0.3cm and 0.4cm,
        img/.style={inner sep=0pt},
        op/.style={draw, rounded corners, fill=gray!10, minimum height=0.5cm, minimum width=0.6cm, font=\scriptsize},
        arrow/.style={->, >=stealth, semithick},
        label/.style={font=\scriptsize, text=black},
    ]

    \node[label] (pz) {$z \sim p_0$};
    \node[op, right=0.35cm of pz, fill=green!10] (pgen) {$G_\phi$};
    \node[img, right=0.4cm of pgen] (pout) {\includegraphics[width=\imgwsm]{figs/pipeline/training_imgs/low_res.png}};
    \node[label, below=0.02cm of pout] {$x_1^{\downarrow}$};

    \draw[arrow] (pz) -- (pgen);
    \draw[arrow] (pgen) -- (pout);

    \node[label, right=\bcgap of pout] (z) {$z \sim \mathcal{N}$};

    \node[op, right=0.3cm of z, fill=green!10] (gen) {$G_\phi$};

    \node[img, right=0.3cm of gen] (lowres) {\includegraphics[width=\imgwsm]{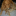}};
    \node[label, below=0.02cm of lowres] {$\hat{x}_1^{\downarrow}$};

    \node[op, right=0.3cm of lowres] (upsample) {$\mathcal{U}$};

    \node[img, right=0.3cm of upsample] (low2high) {\includegraphics[width=\imgw]{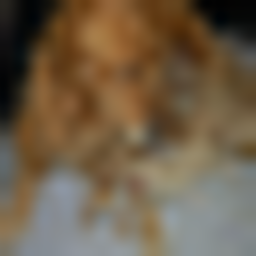}};
    \node[label, below=0.02cm of low2high] {$\mathcal{U}(\hat{x}_1^{\downarrow})$};

    \node[op, right=0.3cm of low2high] (addnoise) {$\alpha$};

    \node[img, right=0.3cm of addnoise] (x0) {\includegraphics[width=\imgw]{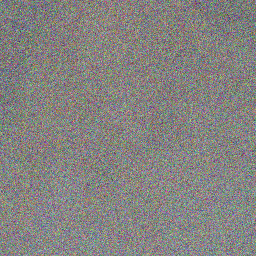}};
    \node[label, below=0.02cm of x0] {$x_0$};

    \node[op, right=0.3cm of x0, fill=blue!10, minimum width=1.3cm, minimum height=0.6cm] (fm) {$v_\theta(x_t, t)$};
    \node[label, below=0.02cm of fm] {\tiny solve ODE};

    \node[img, right=0.3cm of fm] (x1) {\includegraphics[width=\imgw]{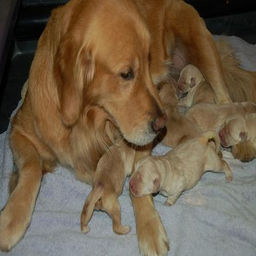}};
    \node[label, below=0.02cm of x1] {$x_1$};

    \draw[arrow] (z) -- (gen);
    \draw[arrow] (gen) -- (lowres);
    \draw[arrow] (lowres) -- (upsample);
    \draw[arrow] (upsample) -- (low2high);
    \draw[arrow] (low2high) -- (addnoise);
    \draw[arrow] (addnoise) -- (x0);
    \draw[arrow] (x0) -- (fm);
    \draw[arrow] (fm) -- (x1);

    \coordinate (bfn_raw) at ([shift={(2pt,\boxpadtop)}]pgen.north -| pout.east);
    \coordinate (cfn_raw) at ([shift={(2pt,\boxpadtop)}]fm.north -| x1.east);
    \path let \p1=(bfn_raw), \p2=(cfn_raw) in
        coordinate (bfn) at (\x1, {max(\y1,\y2)})
        coordinate (cfn) at (\x2, {max(\y1,\y2)});

    \coordinate (bfs) at ([shift={(-2pt,-\boxpadbot)}]pz.south west);
    \draw[dashed, traincolor!60, rounded corners=4pt] (bfs) rectangle (bfn);
    \node[font=\scriptsize\bfseries, text=traincolor!70, fill=white, inner sep=2pt, anchor=south west] at (bfs) {(b)};

    \coordinate (cfs) at ([shift={(-2pt,-\boxpadbot)}]z.south west);
    \draw[dashed, infercolor!60, rounded corners=4pt] (cfs) rectangle (cfn);
    \node[font=\scriptsize\bfseries, text=infercolor!70, fill=white, inner sep=2pt, anchor=south west] at (cfs) {(c)};

    \end{tikzpicture}%
    }%

    \makebox[0pt]{%
      \phantomsubcaption\label{fig:pipeline_a}%
      \phantomsubcaption\label{fig:pipeline_b}%
      \phantomsubcaption\label{fig:pipeline_c}%
    }

    \caption{%
        \textbf{Full Pipeline.}
        \textbf{(a)}~During {\color{traincolor}training}, data samples are projected to low frequencies ($\mathcal{D}$),
        upsampled ($\mathcal{U}$), and perturbed with noise~(\cref{eq:noisy_prior}) to form $x_0$, 
        paired with $x_1$ to train $v_\theta$.
        \textbf{(b)}~$G_\phi$ is {\color{traincolor}trained} to map Gaussian noise to the
        low-frequency distribution.
        \textbf{(c)}~At {\color{infercolor}inference}, $G_\phi$ produces a low-frequency sample used to construct a starting point for generation, solving the ODE using $v_\theta$.
    }
    \label{fig:pipeline}
\end{figure}

%% file: figs/swaps_bars/swaps_bars.tex
\begin{wrapfigure}[18]{R}{0.55\linewidth}
\centering
\definecolor{imgoutline}{HTML}{2E6FAC}
\def\figwidth{0.95\linewidth}  %
\def\imgsize{0.72cm}            %
\def\imgstartX{1.5}           %
\def\imggap{1.255}               %
\def\origX{0.55}                %
\def\imgmargin{-2pt}            %
\def\imgoutlineW{1.5pt}         %
\begin{tikzpicture}[
    img/.style={draw=imgoutline, line width=\imgoutlineW, inner sep=0pt},
    lbl/.style={font=\scriptsize, anchor=north, yshift=-1pt},
    albl/.style={font=\tiny\sffamily, anchor=south, yshift=1pt},
]
\node[inner sep=0pt, anchor=south west] (bc) at (0,0)
    {\includegraphics[width=\figwidth]{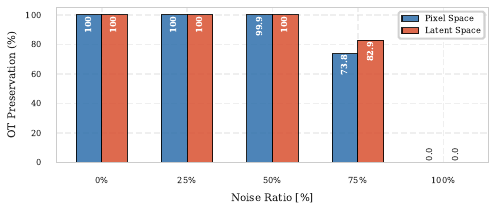}};
\coordinate (imgbase) at ([yshift=\imgmargin]bc.north west);
\coordinate (origpos) at (\origX, 0);
\coordinate (i0pos) at (\imgstartX, 0);
\coordinate (i1pos) at ({\imgstartX+\imggap}, 0);
\coordinate (i2pos) at ({\imgstartX+2*\imggap}, 0);
\coordinate (i3pos) at ({\imgstartX+3*\imggap}, 0);
\coordinate (i4pos) at ({\imgstartX+4*\imggap}, 0);
\node[img, anchor=south] (orig) at (origpos |- imgbase)
    {\includegraphics[width=\imgsize]{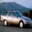}};
\node[albl] at (orig.north) {$x_1$};
\node[img, anchor=south] (i0) at (i0pos |- imgbase)
    {\includegraphics[width=\imgsize]{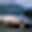}};
\node[albl] at (i0.north) {$\alpha{=}0$};
\node[img, anchor=south] (i1) at (i1pos |- imgbase)
    {\includegraphics[width=\imgsize]{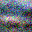}};
\node[albl] at (i1.north) {$\alpha{=}.25$};
\node[img, anchor=south] (i2) at (i2pos |- imgbase)
    {\includegraphics[width=\imgsize]{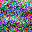}};
\node[albl] at (i2.north) {$\alpha{=}.5$};
\node[img, anchor=south] (i3) at (i3pos |- imgbase)
    {\includegraphics[width=\imgsize]{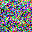}};
\node[albl] at (i3.north) {$\alpha{=}.75$};
\node[img, anchor=south] (i4) at (i4pos |- imgbase)
    {\includegraphics[width=\imgsize]{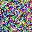}};
\node[albl] at (i4.north) {$\alpha{=}1$};
\end{tikzpicture}
\caption{
    \textbf{OT Preservation Under Noise.}
    Fraction of pairs whose optimal assignment matches the identity coupling,
    in \textit{\textcolor{pixelblue}{pixel space}} (CIFAR-10) and
    \textit{\textcolor{latentred}{latent space}} (FFHQ).
    Above the bars: a sample at each noise level.
}
\label{fig:ot:noise_vs_ot_swaps}
\end{wrapfigure}
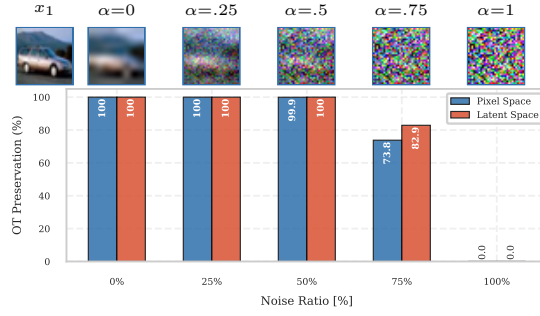

%% file: sec/5_experiments.tex
\section{Experiments}
\label{sec:experiments}
\subsection{Experimental setup}
\label{sec:exp_setup}

\input{figs/cifar_results/cifar_results}\textbf{Datasets.}\enspace
We evaluate on three benchmarks of increasing complexity:
CIFAR-10~\cite{cifar10} ($32 \times 32$),
FFHQ~\cite{karras2019style} ($256 \times 256$),
and ImageNet~\cite{russakovsky2015imagenet} ($256 \times 256$).
CIFAR-10 is trained in pixel space, following the architecture and training recipe of Tong \etal~\cite{tong2024improving},
while
for the $256 \times 256$ datasets,
we follow standard practice and train using the SiT architecture~\cite{geng2025mean,ma2024sit},
and operate in the latent space
of a pretrained autoencoder~\cite{rombach2022stable_diffusion},
yielding $32 \times 32 \times 4$ latents.
We downsample by $4\times$
per spatial dimension,
producing an $8 \times 8$ low-frequency prior.
$G_{\phi}$ is trained on a smaller variant of the same architecture family for each experiment,
and operates on these representations.
Full architectural details, training procedures, low-frequency projection specifications,
and ablations are provided in \suppcref{sec:supp_implementation,sec:supp_ablations}.

\input{figs/cifar_generation/cifar_generation}
\myparagraph{Baselines.}
We compare against
IFM~\cite{lipman2022flow} (independent coupling),
OT-FM~\cite{tong2024improving} (minibatch OT),
and AlignFlow~\cite{kong2025alignflow}
(semi-discrete OT).
All baselines are single-network models,
and our flow model $v_\theta$ uses the same architecture
and parameter count as theirs.
The additional cost of the lightweight $G_\phi$
is reported through the effective NFE,
with full details in \suppcref{tab:cifar_arch,tab:sit_arch,sec:supp_algorithms}.
On ImageNet we evaluate
with classifier-free guidance
(CFG)~\cite{ho2022classifier_free_guidance}
and integrate our prior
into MeanFlow~\cite{geng2025mean},
a one-step generation framework.

\myparagraph{Evaluation.}
We report FID~\cite{fid} for generation quality (log scale in all plots),
and curvature~\cite{pooladian2023multisample,lee2023minimizing},
which quantifies how close the learned trajectories are to straight lines.
Our two-stage pipeline incurs a small overhead
from the lightweight generator $G_\phi$.
In wall-clock time, one $G_\phi$ step takes $\approx 0.09\times$
the time of a main-model step.
We use $8$ generation steps for $G_\phi$
in all experiments (\suppcref{sec:supp_ablation_gen_nfe})
and report \emph{effective NFE} throughout,
accounting for both stages.
For example, $1$ main-model step
yields an effective NFE of $1.72$.
We additionally report results
with a single $G_\phi$ step
(effective NFE overhead of only $0.09$)
to demonstrate robustness
to the prior-generation budget,
and \suppcref{sec:supp_oracle}
provides the perfect-prior upper bound.
Together these bracket the sensitivity of final FID
to $G_\phi$ quality.
We use $\alpha = 0.5$ throughout,
as motivated in \cref{sec:ot_by_design}.
At inference, we additionally apply
a small noise perturbation to $x_0$,
see \suppcref{sec:supp_noise_calibration} for details.

\subsection{Pixel space (CIFAR-10)}
\label{sec:cifar_results}

We begin with CIFAR-10 in pixel space.
\cref{fig:cifar_results_a} plots FID vs NFE for IFM~\cite{lipman2022flow}, OT-FM~\cite{tong2024improving}, AlignFlow~\cite{kong2025alignflow}, and our method.
Note that our effective NFE takes into account the cost of running $G_\phi$.
At all comparison points beyond the first,
our method uses fewer total function evaluations
than the baselines.
Our method outperforms IFM and OT-FM
across all step counts.
Compared to AlignFlow,
which requires a preprocessing stage to solve a semi-discrete OT problem
over the full dataset,
our method achieves better FID
at $4$ and $8$ effective NFE.
As the number of steps increases,
all methods converge,
confirming the advantage of trajectory straightening
is concentrated in the few-step regime.
We additionally report results
with a single $G_\phi$ step
(effective NFE of $1.09$),
demonstrating competitive performance
even in this minimal-cost setting.
\cref{fig:cifar_results_b} explains this improvement through trajectory curvature.
Our method achieves substantially lower curvature
than all baselines,
at least $2\times$ below AlignFlow,
the second-best method,
and $4\times$ lower than IFM.

\cref{fig:cifar_generation} provides qualitative support.
At $1$ NFE, baselines produce
visibly blurred outputs of varying quality,
as the velocity field must compress
curved trajectories into a single step.
Our method, starting from a structured low-frequency prior
($x_0$ column), already captures
the coarse layout of the target image,
and a single step suffices
to synthesize recognizable content.
By $4$ NFE, our samples exhibit
clear object structure and color fidelity,
while the baselines remain noticeably degraded.

\subsection{Latent space (FFHQ \& ImageNet)}
\label{sec:latent_results}
\input{figs/ffhq_results/ffhq_results}
\cref{fig:ffhq_results_a} reports FID
as a function of effective NFE
on FFHQ $256 \times 256$.
Our method achieves the best FID at all step counts
except NFE${} = 2$,
where our effective cost is only $1.72$ steps.
The gains are largest at $1$ step,
with $20\%$ improvement
over OT-FM and $38\%$ over IFM.
AlignFlow~\cite{kong2025alignflow} is excluded
from latent-space experiments,
as it requires precomputed OT assignments
only publicly available for CIFAR-10.
\cref{fig:ffhq_results_b} shows the mechanism
behind these improvements.
Our method reduces curvature to $0.083$,
under half of OT-FM ($0.172$)
and less than $42\%$ of IFM ($0.201$).
This confirms that our approach
extends to latent space,
where the 
OT-identity coupling
holds for the latent representations
under Hungarian verification (\cref{fig:ot:noise_vs_ot_swaps}).

\cref{fig:ffhq_low_nfe} compares samples
generated at $2$ NFE across all three methods.
Notably, our method uses an effective NFE of only $1.72$
at this setting,
yet the visual comparison reveals a clear advantage.
IFM and OT-FM produce faces
with noticeable artifacts
and incomplete facial structure.
Our method generates more coherent samples
with significantly fewer artifacts,
reflecting the advantage of starting
from a structured prior
that already captures the coarse layout.

\input{figs/ffhq_low_nfe/ffhq_low_nfe}
\noindent\textbf{ImageNet and Classifier-Free Guidance.}\enspace
To evaluate scaling to larger datasets,
we train the same SiT backbone
on class-conditional ImageNet $256 \times 256$.
Due to compute constraints,
we train for the same number of iterations
used in our other experiments
rather than to full convergence,
so absolute FID values are higher
than those of fully trained models.
The comparison is nonetheless informative,
as both methods share the same architecture
and training budget.
\cref{tab:imagenet_results} reports FID and curvature
at $1$ NFE ($1.09$ effective for our method).
Since training is partial, absolute FID values are high;
nevertheless, relative improvements are informative.
Our method lowers FID from $314.80$ to $248.42$ with CFG
and from $322.18$ to $268.14$ without CFG,
yielding even larger gains in the guided setting.
Curvature is reduced by more than $2.3\times$
compared to IFM in both configurations,
suggesting that the straighter trajectories
observed on FFHQ extend to a large scale dataset such as ImageNet.
Our approach integrates
with classifier-free guidance~\cite{ho2022classifier_free_guidance},
as both $G_\phi$ and the main model
are trained with the standard CFG scheme.
The FID improvement holds
in both settings,
confirming that our approach
complements CFG rather than substituting for it.

\noindent\textbf{One-Step Generation with MeanFlow.}\enspace
MeanFlow~\cite{geng2025mean} is a one-step
generation framework that trains from scratch
by modeling average rather than instantaneous velocity,
requiring no distillation or curriculum learning.
We integrate our prior into MeanFlow
and evaluate on ImageNet $256 \times 256$,
following the setup of Tab.~1
in the original paper~\cite{geng2025mean}.
We use the pretrained checkpoint from a popular PyTorch reproduction~\cite{meanflow_pytorch}
as our baseline, which achieves better FID than the original paper,
and train our variant from scratch using the same codebase.
As shown in \cref{tab:meanflow_no_cfg},
our method outperforms MeanFlow
at both $1$ and $2$ steps
using only $1.72$ effective NFE, while
reducing curvature by $2.3\times$,
confirming that our prior yields
substantially straighter trajectories
even in a framework
already optimized for one-step generation.
This integration illustrates
how our prior complements training-objective changes
such as consistency models~\cite{pmlr-v202-song23a,yang2024consistency}
and distillation~\cite{liu2022flow,seongbalanced},
which modify the learning objective
while our method modifies the coupling
the velocity field is trained on.

\myparagraph{Ablations.}
We ablate the projection operator (\suppcref{sec:supp_projection_variants}),
the downsampling factor (\suppcref{sec:supp_ablation_downsample}),
and the $G_\phi$ step count (\suppcref{sec:supp_ablation_gen_nfe}),
the first showing that low-frequency structure is the necessary ingredient
beyond OT-optimality alone.

\input{tables/imagenet_and_meanflow}

%% file: figs/cifar_results/cifar_results.tex
\begin{wrapfigure}[14]{R}{0.55\linewidth}
\centering
\vspace{-0.5cm}
\includegraphics[width=\linewidth]{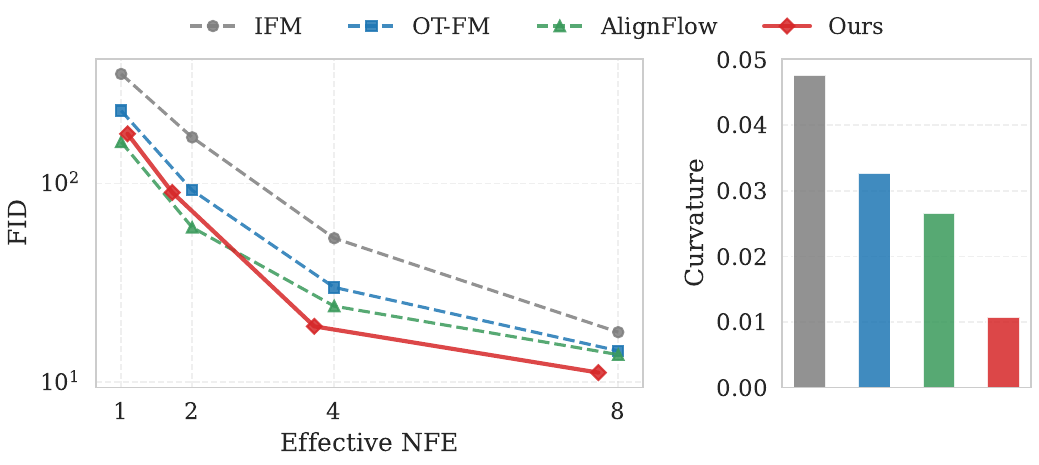}\\[-4pt]
\makebox[\linewidth]{%
  \phantomsubcaption\label{fig:cifar_results_a}
  \makebox[0.7\linewidth]{\small\bfseries (a)}%
  \phantomsubcaption\label{fig:cifar_results_b}
  \makebox[0.3\linewidth]{\small\bfseries (b)}%
}
\caption{\textbf{CIFAR-10 Results.}
Our method achieves lower FID ($\downarrow$) across most step counts \textbf{(a)}
and at least $2\times$ lower curvature ($\downarrow$) than the next best baseline \textbf{(b)}.}
\label{fig:cifar_results}%
\end{wrapfigure}

%% file: figs/cifar_generation/cifar_generation.tex
\begin{wrapfigure}[20]{R}{0.49\linewidth}
\centering
\setlength{\tabcolsep}{0.2pt}
\providecommand{\imgw}{}
\renewcommand{\imgw}{0.08\textwidth}
\providecommand{\img}[1]{}
\renewcommand{\img}[1]{\adjustbox{valign=M}{\includegraphics[width=\imgw]{#1}}}
\providecommand{\rowlbl}[1]{}
\renewcommand{\rowlbl}[1]{\adjustbox{valign=M}{\rotatebox{90}{\txtsize #1}}}
\def\rowsep{1pt}
\def\colsep{1pt}
\def\txtsize{\scriptsize}
\def\insetsep{1pt}
\providecommand{\placeholder}{}
\renewcommand{\placeholder}{\adjustbox{valign=M}{\makebox[\imgw]{\centering\txtsize ---}}}
\vspace{-0.4cm}
\adjustbox{max width=\linewidth}{%
\begin{tabular}{
  c
  @{\hspace{\colsep}} c
  @{\hspace{\colsep}} c
  @{\hspace{\insetsep}} c
  @{\hspace{\colsep}} c
  @{\hspace{\colsep}} c}
& {\txtsize $G_\phi(z)$}
& {\txtsize $x_0$}
& \multicolumn{3}{c}{\txtsize NFE} \\
& & &
{\txtsize 1} &
{\txtsize 2} &
{\txtsize 4} \\[2pt]
\rowlbl{IFM} &
\placeholder &
\img{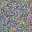} &
\img{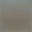} &
\img{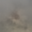} &
\img{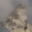} \\
\noalign{\vskip \rowsep}
\rowlbl{OT-FM} &
\placeholder &
\img{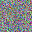} &
\img{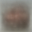} &
\img{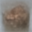} &
\img{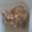} \\
\noalign{\vskip \rowsep}
\rowlbl{AlignFlow} &
\placeholder &
\img{figs/cifar_generation/cifar_low_nfe/gaussian_noise_1.png} &
\img{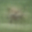} &
\img{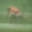} &
\img{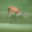} \\
\noalign{\vskip \rowsep}
\rowlbl{Ours} &
\img{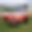} &
\img{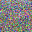} &
\img{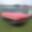} &
\img{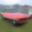} &
\img{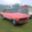} \\
\end{tabular}%
}
\caption{\textbf{CIFAR-10 Generation.}
Samples generated at $1$, $2$, and $4$ NFE.
For our method NFE counts $v_\theta$ evaluations. $G_\phi(z)$ is the generated low-frequency sample
and $x_0$ the noised starting point.
All baselines start from Gaussian noise.}
\label{fig:cifar_generation}
\end{wrapfigure}

%% file: figs/ffhq_results/ffhq_results.tex
\begin{wrapfigure}[15]{R}{0.53\linewidth}
\centering
\vspace{-0.5cm}
\includegraphics[width=\linewidth]{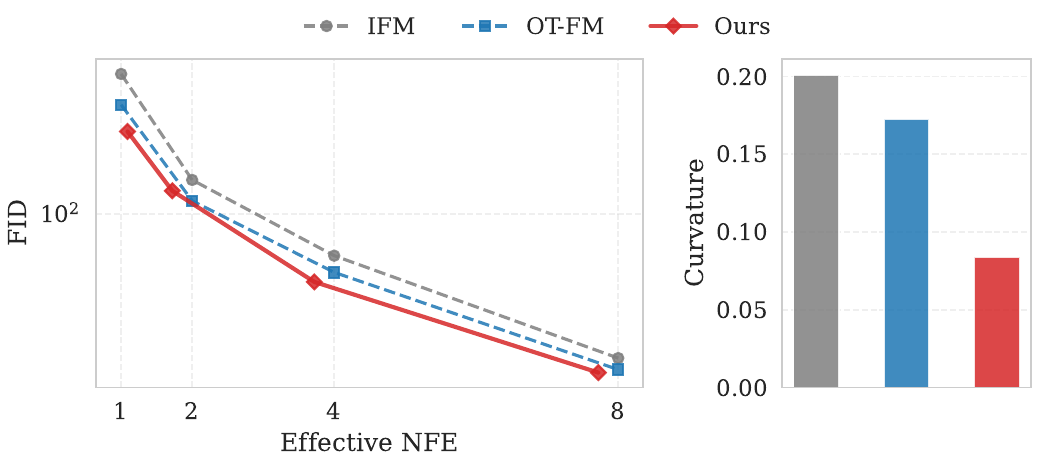}\\[-4pt]
\def\posA{3.1cm}
\def\posB{4.7cm}
\makebox[\linewidth]{%
  \phantomsubcaption\label{fig:ffhq_results_a}%
  \makebox[0.7\linewidth]{\small\bfseries (a)}%
  \phantomsubcaption\label{fig:ffhq_results_b}%
  \makebox[0.3\linewidth]{\small\bfseries (b)}%
}
\caption{
    \textbf{FFHQ Results.}
    \textbf{(a)} Our method achieves lower FID ($\downarrow$) across all steps but $2$.
    \textbf{(b)} Our method achieves at least $2\times$ lower curvature ($\downarrow$)
    than the next best baseline.}
\label{fig:ffhq_results}%
\end{wrapfigure}

%% file: figs/ffhq_low_nfe/ffhq_low_nfe.tex
\begin{figure}[ht]
\centering
\setlength{\tabcolsep}{0.1pt}
\providecommand{\imgw}{}
\renewcommand{\imgw}{0.088\textwidth}
\providecommand{\img}[1]{}
\renewcommand{\img}[1]{\adjustbox{valign=M}{\includegraphics[width=\imgw]{#1}}}
\providecommand{\lblsize}{}
\renewcommand{\lblsize}{\small}
\providecommand{\rowlbl}[1]{}
\renewcommand{\rowlbl}[1]{\adjustbox{valign=M}{\rotatebox{90}{\lblsize #1}}}
\ifdefined\rowsep\else\newlength{\rowsep}\fi\setlength{\rowsep}{13pt}
\ifdefined\colsep\else\newlength{\colsep}\fi\setlength{\colsep}{0.5pt}
\newcommand{\imgdir}{figs/ffhq_low_nfe/images}
\adjustbox{max width=\linewidth}{%
\begin{tabular}{c @{\hspace{4\colsep}} c @{\hspace{\colsep}} c @{\hspace{\colsep}} c @{\hspace{\colsep}} c @{\hspace{\colsep}} c @{\hspace{\colsep}} c @{\hspace{\colsep}} c @{\hspace{\colsep}} c @{\hspace{\colsep}} c @{\hspace{\colsep}} c}
\rowlbl{IFM} &
\img{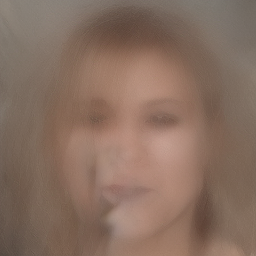} &
\img{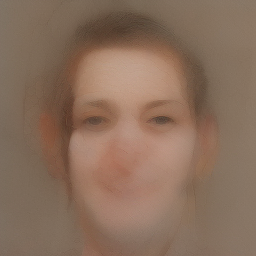} &
\img{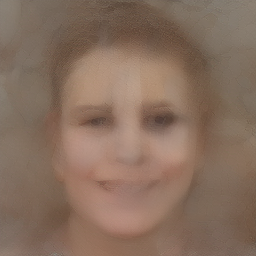} &
\img{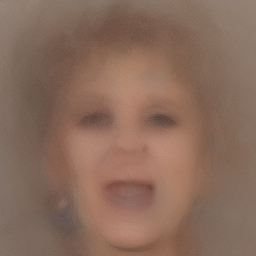} &
\img{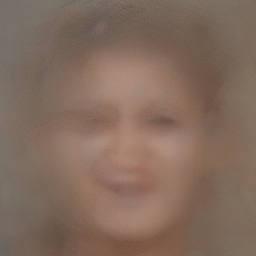} &
\img{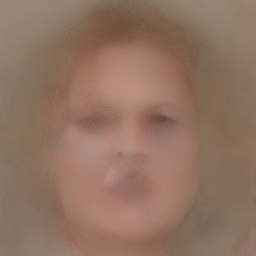} &
\img{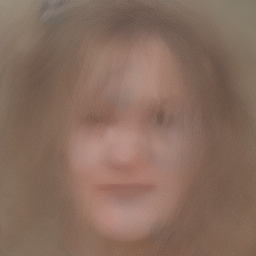} &
\img{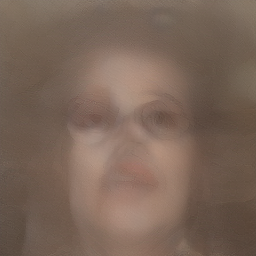} &
\img{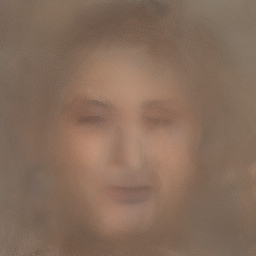} &
\img{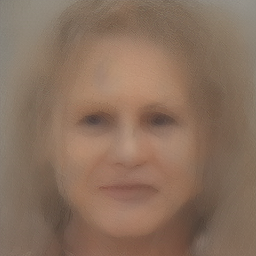} \\[\rowsep]
\rowlbl{OT-FM} &
\img{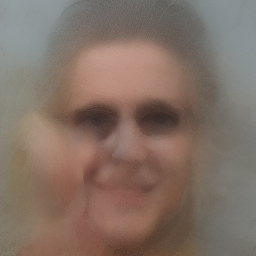} &
\img{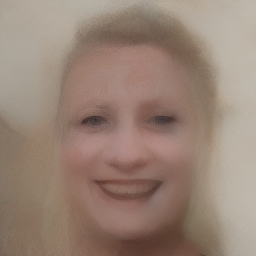} &
\img{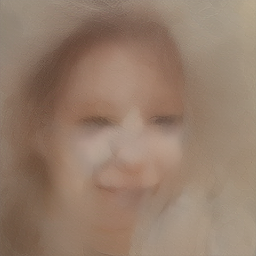} &
\img{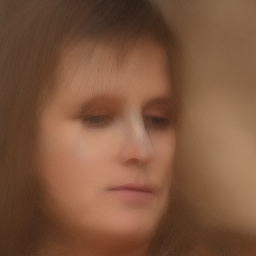} &
\img{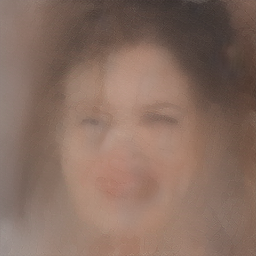} &
\img{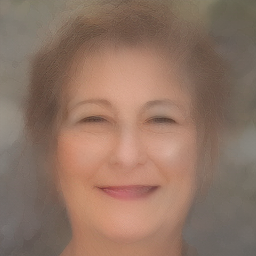} &
\img{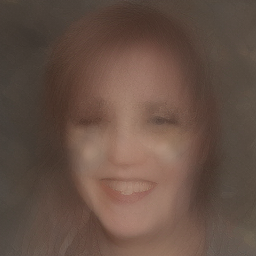} &
\img{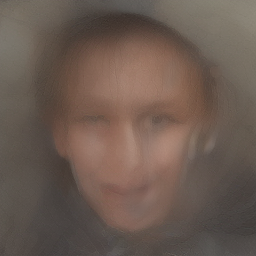} &
\img{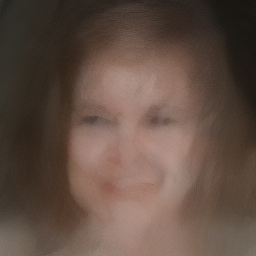} &
\img{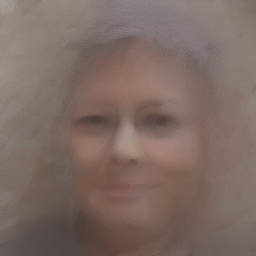} \\[\rowsep]
\rowlbl{Ours} &
\img{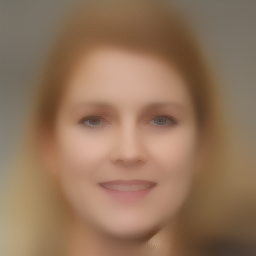} &
\img{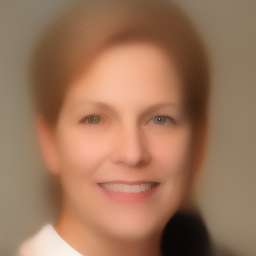} &
\img{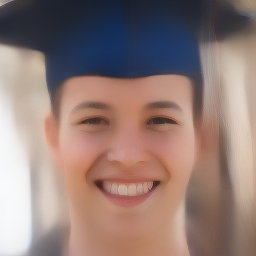} &
\img{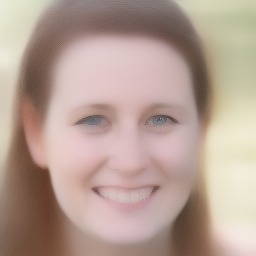} &
\img{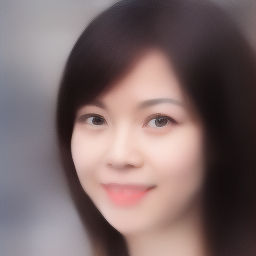} &
\img{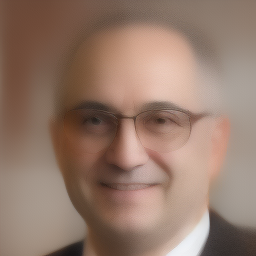} &
\img{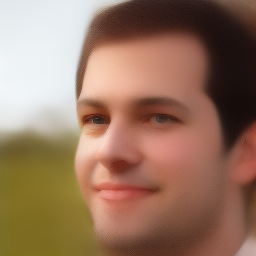} &
\img{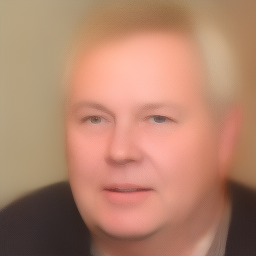} &
\img{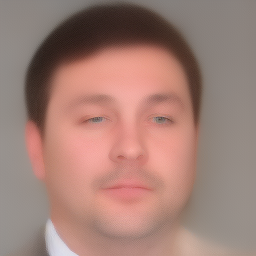} &
\img{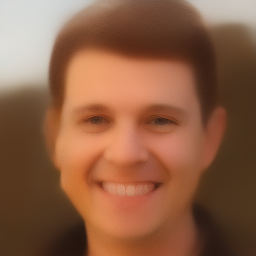} \\
\end{tabular}%
}
\caption{\textbf{FFHQ 2-NFE generation.}
Baselines use $2$ NFE,
while our method uses an effective NFE of only $1.72$,
\emph{fewer} than the baselines.
}
\label{fig:ffhq_low_nfe}
\end{figure}

%% file: tables/imagenet_and_meanflow.tex
\begin{table}[ht]
\centering
\begin{minipage}[t]{0.48\linewidth}
  \centering
  \small
  \setlength{\tabcolsep}{4pt}
  \renewcommand{\arraystretch}{1.15}
  \captionof{table}{\textbf{ImageNet CFG Results.}
  FID and curvature
  at $1$ NFE ($1.09$ effective for our method).
  High FIDs reflect partial training.}
  \label{tab:imagenet_results}

  \resizebox{0.8\linewidth}{!}{%
  \begin{tabular}{llcc}
    \toprule
    & Method & FID ($\downarrow$) & Curv. ($\downarrow$) \\
    \midrule
    \multirow{2}{*}{w/o CFG}
      & IFM & 322.18 & 0.227 \\
      & \cellcolor{gray!10} Ours & \cellcolor{gray!10} \textbf{268.14} & \cellcolor{gray!10} \textbf{0.097} \\
    \addlinespace[3pt]
    \multirow{2}{*}{w/ CFG}
      & IFM & 314.80 & 0.223 \\
      & \cellcolor{gray!10} Ours & \cellcolor{gray!10} \textbf{248.42} & \cellcolor{gray!10} \textbf{0.095} \\
    \bottomrule
  \end{tabular}}
\end{minipage}
\hfill
\begin{minipage}[t]{0.48\linewidth}
  \centering
  \small
  \setlength{\tabcolsep}{3pt}
  \renewcommand{\arraystretch}{1.0}
  \captionof{table}{\textbf{MeanFlow on ImageNet.}
  FID and curvature at low-NFE.
  High FIDs reflect partial training.}
  \label{tab:meanflow_no_cfg}
  \resizebox{0.8\linewidth}{!}{%
  \begin{tabular}{lccc}
    \toprule
    Method & NFE & FID ($\downarrow$) & Curv. ($\downarrow$) \\
    \midrule
    \multirow{2}{*}{MeanFlow} & 1 & 59.07 & \multirow{2}{*}{0.221} \\
      & 2 & 55.40 & \\
    \addlinespace[3pt]
    \cellcolor{gray!10} \makecell[l]{MeanFlow \\ $+$ Ours} & \cellcolor{gray!10} 1.72 & \cellcolor{gray!10} \textbf{54.66} & \cellcolor{gray!10} \textbf{0.097} \\
    \bottomrule
  \end{tabular}}
\end{minipage}

\end{table}

%% file: sec/6_limitations.tex
\section{Limitations and conclusion}
\label{sec:limitations}
\label{sec:conclusion}

\myparagraph{Prior generation.}
Our empirical OT-identity coupling applies to the flow
from $\tilde{p}_0$ to the data,
but not to generating $\tilde{p}_0$ itself,
as $G_\phi$ is trained under independent coupling,
so the entire pipeline
is not fully OT-coupled.
This adds only a small inference overhead in practice,
since low-frequency natural-image distributions
are smooth and concentrated (\cref{sec:method}),
making them easier to generate~\cite{ho2022cascaded}.
At the same time, $G_\phi$ remains the main bottleneck:
in an oracle CIFAR-10 experiment,
replacing $G_\phi$ with true low-frequency projections
of training images improves FID at $1$ NFE
from $90.16$ to $28.12$.
Since $G_\phi$ is decoupled from $v_\theta$,
improving it should directly improve the full pipeline
(see \suppcref{sec:supp_oracle}).

\myparagraph{Domain specificity.}
Our prior uses low-frequency projections,
which empirically preserve the OT-identity coupling for natural images
because their spectrum is concentrated at low frequencies~\cite{ruderman1994statistics},
keeping projected samples well separated.
This property also carries over to VAE latents (\cref{sec:latent_results})
and holds for structured datasets such as FFHQ.
For other modalities,
whether an analogous projection preserves the OT-identity coupling
depends on the data geometry and should be examined per domain.

Overall, our results show that OT couplings for flow matching
can be obtained not only by solving an OT problem,
but also by designing a suitable prior.
Replacing the Gaussian prior with low-frequency projections
yields an empirically OT-optimal identity coupling,
straighter trajectories,
and improved few-step generation quality,
while integrating naturally with latent-space models,
classifier-free guidance,
and one-step generation frameworks.

%% file: sec/7_1_acks.tex
\section{Acknowledgements}
This project was partially supported by ISF grants 1574/21 and 2132/23.

%% file: sec/8_sup.tex
\clearpage
\appendix
\setcounter{figure}{0}
\setcounter{table}{0}
\setcounter{equation}{0}
\setcounter{algorithm}{0}
\renewcommand{\thefigure}{\Alph{section}.\arabic{figure}}
\renewcommand{\thetable}{\Alph{section}.\arabic{table}}
\renewcommand{\theequation}{\Alph{section}.\arabic{equation}}
\renewcommand{\thealgorithm}{\Alph{section}.\arabic{algorithm}}
\renewcommand{\theproposition}{\Alph{section}.\arabic{proposition}}

\ifarxiv
  \section*{Appendices}
\fi

We provide
implementation details (\cref{sec:supp_implementation}),
empirical validation of the OT coupling (\cref{sec:supp_empirical}),
ablation studies (\cref{sec:supp_ablations}),
formal proofs (\cref{sec:supp_proofs}),
and qualitative results (\cref{sec:supp_visualizations}).

\input{sec/supp/technical}

\input{sec/supp/analysis}
\input{sec/supp/ablation}

\input{sec/supp/proofs}
\input{sec/supp/qualitative}

%% file: sec/supp/technical.tex
\section{Implementation details}
\label{sec:supp_implementation}

\subsection{Architecture and training}
\label{sec:supp_algorithms}

\input{tables/supp_pixel_architecture}
\input{tables/supp_latent_architecture}
\myparagraph{CIFAR-10.}
\cref{tab:cifar_arch} lists the architecture
and training configuration for CIFAR-10 (\mcref{sec:cifar_results}).
We train IFM~\mcite{lipman2022flow},
OT-FM~\mcite{tong2024improving},
AlignFlow~\mcite{kong2025alignflow},
and our method from scratch,
except that for AlignFlow
we additionally use the pretrained assignment model
from the official repository.

\input{algs/coupling}

\myparagraph{FFHQ, ImageNet, and MeanFlow.}
\cref{tab:sit_arch} lists the configuration
for the latent-space experiments.
All methods are trained from scratch,
except for MeanFlow~\mcite{geng2025mean}
where we use the pretrained checkpoint
from a popular PyTorch reproduction~\mcite{meanflow_pytorch}
as our baseline.
Training is partial in all latent-space experiments.
On ImageNet, we train for $80$ epochs,
matching the ablation budget
in Tab.~1 of the original paper~\mcite{geng2025mean}.
On FFHQ, we use a comparable number of training steps.
For the classifier-free guidance (CFG) setting,
the same iteration count is reached
with a larger batch size for efficiency,
and we use a guidance scale of $1.5$
at inference following prior work~\mcite{ma2024sit}.

\myparagraph{Prior construction.}
In all experiments,
we set $\alpha = 0.5$ (\mcref{eq:noisy_prior})
for training $v_\theta$.
The low-frequency prior is obtained by downsampling
by a factor of $4$ in each spatial dimension,
applied to pixels for CIFAR-10
and to latents for FFHQ and ImageNet.
\cref{alg:coupling} provides a PyTorch-style pseudocode implementation
of the training and inference procedures.

\myparagraph{Inference overhead.}
Our two-stage pipeline introduces a small fixed overhead
from the lightweight generator $G_\phi$,
which operates on $8 \times 8$ representations
using a smaller variant of the main architecture
(for SiT models, SiT-S/2 with $16$ tokens,
compared to SiT-B/4 with $64$ tokens for $v_\theta$).
\cref{fig:timing} reports the wall-clock time per step
for both models across batch sizes.
One $G_\phi$ step consistently takes
$\approx 0.08$--$0.09\times$ the time of a $v_\theta$ step,
so the $8$ generation steps used for $G_\phi$
amount to $\approx 0.72$ effective NFE of overhead.
This cost is fixed regardless of the number of main-model steps,
meaning the relative overhead shrinks
as NFE increases.
\input{figs/supp/timing/timing_fig}

\subsection{Compute resources}
\label{sec:supp_compute}

Experiments were run on a heterogeneous mix of NVIDIA A5000, A6000, and H100 GPUs.
For a uniform reporting unit, we report the wall-clock training time of a single run on one A5000-equivalent GPU (training only, excluding evaluation), separately for the main flow model $v_\theta$ and the lightweight low-frequency generator $G_\phi$:
CIFAR-10: $v_\theta$ $\approx 28$\,h, $G_\phi$ $\approx 1.5$\,h;
FFHQ-256: $v_\theta$ $\approx 25$\,h, $G_\phi$ $\approx 2.5$\,h;
MeanFlow ImageNet-256: $v_\theta$ $\approx 260$\,h, $G_\phi$ $\approx 24$\,h.
$G_\phi$ thus adds $\sim$5-10\% on top of the main-model training cost.

\subsection{Curvature}
\label{sec:supp_curvature}

To quantify trajectory straightness,
we report the curvature metric
used in prior work~\mcite{lipman2022flow,lee2023minimizing},
which measures the deviation of the learned trajectory
from a straight path,
\begin{equation}
    \label{eq:curvature}
    \mathcal{C} = \mathbb{E}_{t}\left[\left\| (x_1 - x_0) - v_\theta(x_t, t)\right\|^2\right],
\end{equation}
where $v_\theta(x_t, t)$ is the learned velocity
and $x_1 - x_0$ is the constant velocity
of the straight-line path.
A perfectly straight trajectory has zero curvature.
Following Lee~\etal~\mcite{lee2023minimizing},
we evaluate on $10{,}000$ generated samples
using $128$ Euler steps.

\subsection{Inference-time noise calibration}
\label{sec:supp_noise_calibration}

At inference, the starting point
is constructed from $\mathcal{U}(\hat{x}_1^{\downarrow})$
by adding Gaussian noise at level $\alpha$
(\mcref{fig:pipeline_c}), according to \mcref{eq:noisy_prior}.
During training,  this is set to $\alpha = 0.5$.
At inference, we find that a slightly higher noise level
$\alpha^* = \alpha + \varepsilon$,
with $\varepsilon > 0$,
consistently improves FID,
particularly in the few-step regime.
This is consistent with observations
in score-based generative modeling~\mcite{song2019generative},
where slightly expanding the support
of a concentrated distribution
improves the conditioning of the learned vector field
and reduces artifacts at inference time.

Across all benchmarks,
the optimal inference noise follows
$\alpha^* = \alpha_{\text{train}} + \varepsilon$,
where $\varepsilon > 0$ is a small excess
that decreases with the number of function evaluations.
In the few-step regime (up to $4$ effective NFE),
we find $\varepsilon \approx 0.05$
in both pixel space (CIFAR-10)
and latent space (FFHQ, ImageNet),
consistent across datasets and resolutions.
For MeanFlow,
the effect is milder ($\varepsilon \approx 0.01$),
consistent with its single-step formulation
being less sensitive to the source distribution.

%% file: tables/supp_pixel_architecture.tex
\begin{table}[ht]
\centering
\footnotesize
\setlength{\tabcolsep}{4pt}
\renewcommand{\arraystretch}{1.12}
\caption[CIFAR-10 architecture and training.]{%
\textbf{CIFAR-10 Architecture and Training.}
Models share the UNet backbone, as prior work~\mcite{lipman2022flow,tong2024improving},
differing only in capacity (bottom block).}
\label{tab:cifar_arch}

\begin{tabular}{lcc}
\toprule
\multicolumn{3}{l}{\cellcolor{gray!8}\textit{Shared architecture}} \\
\addlinespace[1pt]
Architecture          & \multicolumn{2}{c}{UNet} \\
Base channels         & \multicolumn{2}{c}{128} \\
Res.\ blocks / level  & \multicolumn{2}{c}{2} \\
Attn.\ heads          & \multicolumn{2}{c}{4\;(64 dim/hd)} \\
Dropout               & \multicolumn{2}{c}{0.1} \\
\addlinespace[3pt]
\multicolumn{3}{l}{\cellcolor{gray!8}\textit{Shared training}} \\
\addlinespace[1pt]
Optimizer             & \multicolumn{2}{c}{Adam\;($2\!\times\!10^{-4}$)} \\
LR warmup             & \multicolumn{2}{c}{5k (linear)} \\
Batch size            & \multicolumn{2}{c}{128} \\
Total steps           & \multicolumn{2}{c}{400k} \\
EMA decay             & \multicolumn{2}{c}{0.9999} \\
Grad.\ clip           & \multicolumn{2}{c}{1.0} \\
Augmentation          & \multicolumn{2}{c}{Horiz.\ flip} \\
\midrule
\cellcolor{gray!8} & \cellcolor{gray!8}$v_\theta$ & \cellcolor{gray!8}$G_\phi$ \\
\addlinespace[1pt]
Input res.            & $32^2\!\times\!3$ & $8^2\!\times\!3$ \\
Chan.\ mult.          & [1,2,2,2] & [1,2] \\
Attn.\ res.           & 16 & 4 \\
Parameters            & ${\sim}$36\,M & ${\sim}$10\,M \\
\bottomrule
\end{tabular}
\end{table}

%% file: tables/supp_latent_architecture.tex
\begin{table}[ht]
\centering
\footnotesize
\setlength{\tabcolsep}{4pt}
\renewcommand{\arraystretch}{1.12}
\caption[256$\times$256 architecture and training.]{%
\textbf{256$\times$256 Architecture and Training.}
All setups share SiT~\mcite{ma2024sit} backbones for $v_\theta$ and $G_\phi$.
MeanFlow adds dual time embedding ($r$, $t$).}
\label{tab:sit_arch}

\begin{tabular}{lccc}
\toprule
\multicolumn{4}{l}{\cellcolor{gray!8}\textit{Shared training}} \\
\addlinespace[1pt]
Optimizer       & \multicolumn{3}{c}{AdamW\;($1\!\times\!10^{-4}$, const.)} \\
EMA decay       & \multicolumn{3}{c}{0.9999} \\
\addlinespace[3pt]
\rowcolor{gray!8}
 & $v_\theta$ (B/4) & \multicolumn{2}{c}{$G_\phi$ (S/2)} \\
\addlinespace[1pt]
Depth/hid./heads & 12/768/12 & \multicolumn{2}{c}{12/384/6} \\
Patch / tokens   & 4\,/\,64 & \multicolumn{2}{c}{2\,/\,16} \\
Input            & $32^2\!\times\!4$ & \multicolumn{2}{c}{$8^2\!\times\!4$} \\
Parameters       & ${\sim}$130\,M & \multicolumn{2}{c}{${\sim}$33\,M} \\
\midrule
\cellcolor{gray!8}
  & \cellcolor{gray!8}\textit{FFHQ}
  & \cellcolor{gray!8}\textit{ImageNet}
  & \cellcolor{gray!8}\textit{MeanFlow} \\
\addlinespace[1pt]
Time emb.       & \multicolumn{2}{c}{Single ($t$)} & Dual ($r$,$t$) \\
Betas            & \multicolumn{2}{c}{(.9,\,.999)} & (.9,\,.95) \\
Time sampler     & \multicolumn{2}{c}{Uniform} & Logit-norm. \\
$r\!\neq\!t$ ratio & \multicolumn{2}{c}{---} & 0.25 \\
Grad.\ clip      & \multicolumn{2}{c}{---} & 1.0 \\
Batch size       & 256 & 1024 & 256 \\
Epochs           & 1400 & 80 & 80 \\
Steps            & 380k & 100k & 400k \\
\bottomrule
\end{tabular}
\end{table}

%% file: algs/coupling.tex
\begin{algorithm}[t]
\begin{algorithmic}[1]
\Statex \textbf{Hyperparameters:} noise ratio $\alpha = 0.5$, downsampling factor $k = 4$
\Statex
\Statex \Comment{\color{traincolor}\textbf{Training: construct OT-structured pairs}}
\Function{BuildCoupling}{$x_1$}
    \State $x_1^{\downarrow} \gets$ \texttt{F.interpolate($x_1$, scale\_factor=1/$k$, mode="area")}
    \State $\mathcal{U}(x_1^{\downarrow}) \gets$ \texttt{F.interpolate($x_1^{\downarrow}$, scale\_factor=$k$, mode="bicubic")}
    \State $\epsilon \gets$ \texttt{torch.randn\_like($\mathcal{U}(x_1^{\downarrow})$)}
    \State $x_0 \gets (1 - \alpha)\, \mathcal{U}(x_1^{\downarrow}) + \alpha\, \epsilon$ \Comment{\mcref{eq:noisy_prior}}
    \State \Return $x_0$
\EndFunction
\Statex
\Statex \Comment{\color{traincolor}\textbf{Training loop}}
\For{$x_1 \sim p_1$}
    \State $x_0 \gets$ \Call{BuildCoupling}{$x_1$}
    \State $t \gets$ \texttt{torch.rand(1)}
    \State $x_t \gets t\, x_1 + (1 - t)\, x_0$ \Comment{\mcref{eq:fm_path}}
    \State $\mathcal{L} \gets \|v_\theta(x_t, t) - (x_1 - x_0)\|^2$ \Comment{\mcref{eq:fm_loss}}
    \State Update $\theta$
\EndFor
\Statex
\Statex \Comment{\color{infercolor}\textbf{Inference}}
\Function{Generate}{$G_\phi$, $v_\theta$}
    \State $z \gets$ \texttt{torch.randn(...)}\Comment{sample Gaussian noise}
    \State $\hat{x}_1^{\downarrow} \gets G_\phi(z)$ \Comment{generate low-frequency sample}
    \State $\mathcal{U}(\hat{x}_1^{\downarrow}) \gets$ \texttt{F.interpolate($\hat{x}_1^{\downarrow}$, scale\_factor=$k$, mode="bicubic")}
    \State $\epsilon \gets$ \texttt{torch.randn\_like($\mathcal{U}(\hat{x}_1^{\downarrow})$)}
    \State $\tilde{x}_0 \gets (1 - \alpha)\, \mathcal{U}(\hat{x}_1^{\downarrow}) + \alpha\, \epsilon$
    \State $x_1 \gets \textsc{OdeSolve}(v_\theta, \tilde{x}_0, t{=}0 \to 1)$
    \State \Return $x_1$
\EndFunction
\end{algorithmic}
\caption{\textbf{Training and Inference.}}
\label{alg:coupling}
\end{algorithm}

%% file: figs/supp/timing/timing_fig.tex
\begin{figure}[htbp]
\centering
\includegraphics[width=0.6\linewidth]{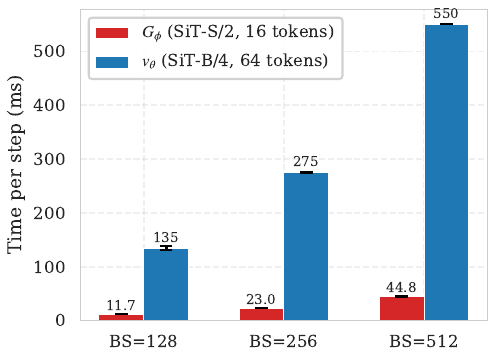}
\caption{\textbf{$v_\theta$ vs. $G_\phi$ Wall Time.}
Inference time per step (ms) for the low-frequency generator $G_\phi$ (SiT-S/2, 16 tokens)
and the main model $v_\theta$ (SiT-B/4, 64 tokens)
across different batch sizes.
One $G_\phi$ step is $\approx$ $0.08$--$0.09$ steps of $v_\theta$.}
\label{fig:timing}
\end{figure}

%% file: sec/supp/analysis.tex
\section{Empirical validation}
\label{sec:supp_empirical}

\subsection{OT preservation under noise}
\label{sec:supp_pairwise}
This section provides extended details
for the experiment in \mcref{fig:ot:noise_vs_ot_swaps},
including the FID-vs-$\alpha$ and pairwise-distance plots.
\input{figs/supp/full_ot_swaps/swaps_fig}

\myparagraph{OT coupling preservation.}
In \mcref{fig:ot:noise_vs_ot_swaps},
we measure the fraction of training pairs
for which the exact discrete OT assignment
(computed using the POT library)
coincides with the identity coupling,
as the noise level $\alpha$ in \mcref{eq:noisy_prior} increases.
The experiment uses $10{,}000$ CIFAR-10 samples in pixel space
and $10{,}000$ FFHQ samples
in the latent representation of a pretrained
autoencoder~\mcite{rombach2022stable_diffusion,ma2024sit}.
\cref{fig:noise_vs_ot_preservation} extends this analysis
to finer noise-ratio increments,
confirming that the identity coupling
is fully preserved up to $\alpha \approx 0.5$.

\input{figs/supp/fid_pairwise/fig}

\myparagraph{FID vs.\ noise ratio.}
\cref{fig:ot:fid_vs_noise} reports FID
on CIFAR-10 as a function of $\alpha$,
evaluated at $2$ NFE
after an early training stage of $40{,}000$ steps
($10\%$ of full training).
Quality improves from $\alpha = 0$ to $\alpha = 0.5$,
then degrades as noise overwhelms
the low-frequency structure.

\myparagraph{Pairwise distance analysis.}
\cref{fig:ot:pairwise_dist_noise}
(discussed in \mcref{sec:ot_by_design})
shows the mean pairwise $\ell_2$ distance
between prior samples as a function of $\alpha$ (proportional to noise ratio).
At $\alpha = 0$,
the mean pairwise distance between
low-frequency projections on CIFAR-10
is ${\approx}\,5.4 \times 10^{-3}$
(normalized by $d = 3{,}072$),
roughly $11\%$ below the original-image distance
of ${\approx}\,6.1 \times 10^{-3}$.
This concentration is consistent
with the $1/f^2$ spectral decay
of natural images~\mcite{field1987relations},
which places the majority of image energy
in the lowest frequencies.
As $\alpha$ increases,
the noise term dominates
and the mean pairwise distance grows steadily.

In the main text (\mcref{eq:noisy_distance}),
we state that the cross term vanishes
and the expected distance decomposes
into a projection term and a noise term.
We now show this step by step.
Expanding the squared distance
between two noisy prior samples
$x_0^{(i)}$ and $x_0^{(j)}$
constructed via \mcref{eq:noisy_prior}.
Expanding the squared norm,
\begin{equation}
\begin{split}
    \norm{x_0^{(i)} - x_0^{(j)}}^2
    &= \norm{(1 - \alpha)\bigl(T(x^{(i)}) - T(x^{(j)})\bigr)
       + \alpha\bigl(\epsilon^{(i)} - \epsilon^{(j)}\bigr)}^2.
\end{split}
\end{equation}
Since the noise terms are independent of the data
and of each other,
the cross term vanishes in expectation,
\begin{equation}
\begin{split}
    \mathbb{E}\bigl[\norm{x_0^{(i)} - x_0^{(j)}}^2\bigr]
    &= (1 - \alpha)^2\,
       \mathbb{E}\bigl[\norm{T(x^{(i)}) - T(x^{(j)})}^2\bigr]
       + 2\alpha^2 d,
\end{split}
\end{equation}
where $d$ is the data dimensionality
and the last term follows from
$\epsilon^{(i)} - \epsilon^{(j)} \sim \mathcal{N}(0, 2I_d)$.
The first term captures the pairwise spread
in the projected subspace,
scaled down by $(1 - \alpha)^2$,
while the second grows linearly with $d$,
explaining the steady increase
visible in \cref{fig:ot:pairwise_dist_noise}.

\subsection{Oracle experiment}
\label{sec:supp_oracle}
\input{figs/oracle/oracle_fig}

To isolate the effect of $G_\phi$
on generation quality,
we evaluate two oracle conditions
alongside our full pipeline on CIFAR-10.
Both bypass $G_\phi$ entirely,
constructing starting points
directly from training images.

\myparagraph{Oracle conditions.}
\textit{Oracle prior} uses our standard training procedure
($\alpha = 0.5$)
but replaces $G_\phi$ at inference
with exact prior samples as in~\mcref{eq:noisy_prior}
derived from training images.
\textit{Deterministic oracle} trains with $\alpha = 0$
(fully deterministic coupling, $x_0 = T(x_1)$)
and starts inference from the exact projections $T(x_i)$.
\cref{fig:cifar_oracle_ablation} reports
FID and curvature for all conditions.

\myparagraph{Results.}
Both oracle conditions achieve lower FID
than the full pipeline across all step counts,
confirming that improving $G_\phi$
can further close the gap to oracle-level performance.
The deterministic oracle achieves the best FID
and the lowest curvature overall,
consistent with the theoretical optimality
of the projection-based coupling.

An interesting exception is that
the full pipeline with $G_\phi$
achieves \emph{lower} curvature
than the noised oracle,
despite worse FID.
We hypothesize that this is because
the noised oracle starts near specific training projections,
so when two projections $T(x_i)$ and $T(x_j)$ are close,
the velocity field must reconcile
conflicting transport directions
within a small neighborhood,
increasing curvature.
$G_\phi$, by contrast, produces samples
that land between training projections
rather than near any specific one,
in regions where the velocity field
may be smoother and more internally consistent.

%% file: figs/supp/full_ot_swaps/swaps_fig.tex
\begin{figure}[htbp]
    \centering
    \includegraphics[width=0.7\linewidth]{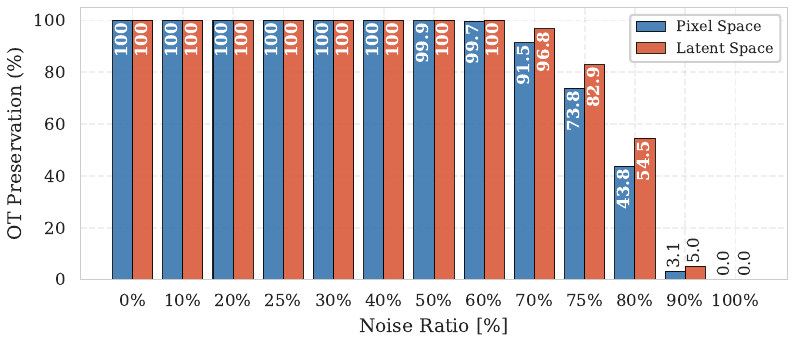}
    \caption{\textbf{Comparison of noise ratio vs. OT preservation.} 
    An elaborated version of \protect\mcref{fig:ot:noise_vs_ot_swaps}
     showing the OT preservation metric as a function of the noise ratio $\alpha$ in \protect\mcref{eq:noisy_prior}.}
    \label{fig:noise_vs_ot_preservation}
\end{figure}

%% file: figs/supp/fid_pairwise/fig.tex
\begin{figure}[htbp]
\centering
\begin{subfigure}[t]{0.48\linewidth}
    \centering
    \includegraphics[width=\linewidth]{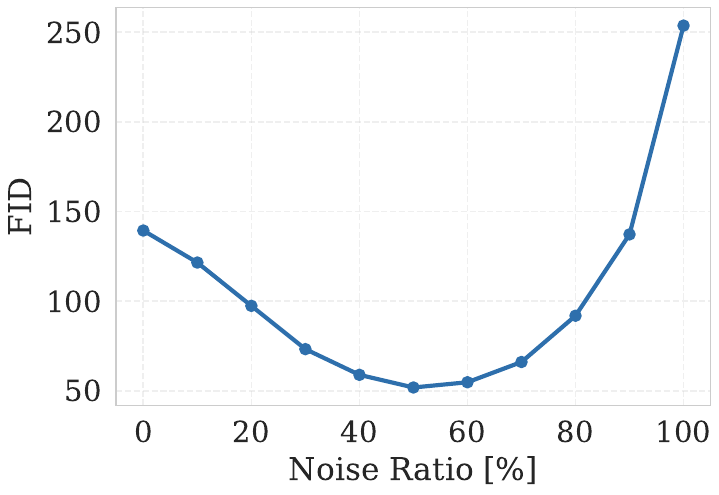}
    \caption{}
    \label{fig:ot:fid_vs_noise}
\end{subfigure}\hfill
\begin{subfigure}[t]{0.48\linewidth}
    \centering
    \includegraphics[width=\linewidth]{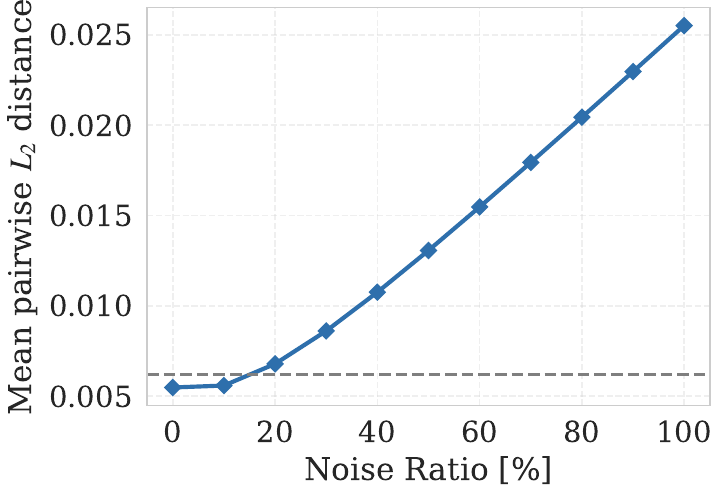}
    \caption{}
    \label{fig:ot:pairwise_dist_noise}
\end{subfigure}
\caption{%
    \textbf{Effect of noise ratio $\alpha$ on generation quality and prior spread.}
    \textbf{(a)}~FID on CIFAR-10 (2 NFE) vs.\ noise ratio.
    \textbf{(b)}~Pairwise $\ell_2$ distance on CIFAR-10 vs.\ noise ratio;
    \textit{\textcolor{gray}{dashed}}: original images.}
\label{fig:supp:fid_pairwise}
\end{figure}

%% file: figs/oracle/oracle_fig.tex
\begin{figure}[t]
\centering
\includegraphics[width=0.7\linewidth]{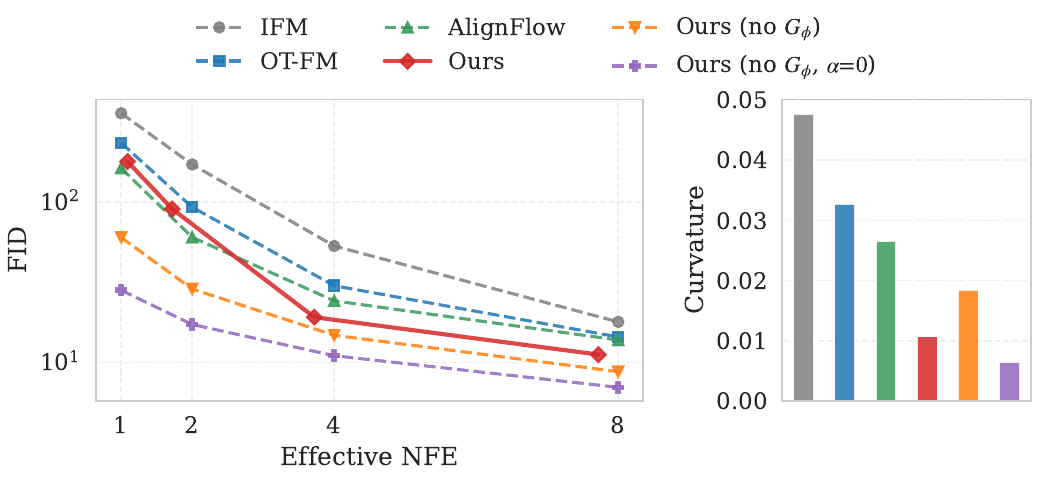}\\[-4pt]
\setcounter{subfigure}{0}%
\makebox[\linewidth]{%
  \makebox[0.5\linewidth]{\small (a)\refstepcounter{subfigure}\label{fig:cifar_oracle_ablation_a}}%
  \makebox[0.2\linewidth]{\small (b)\refstepcounter{subfigure}\label{fig:cifar_oracle_ablation_b}}%
}
\caption{\textbf{Oracle Experiment on CIFAR-10.}
FID ($\downarrow$) vs.\ effective NFE
and trajectory curvature~($\downarrow$)
for baselines, our full pipeline,
and two oracle conditions
that bypass $G_\phi$ using true training-image projections.
Panels (a) and (b) show FID and curvature, respectively.
\textit{Ours (no $G_\phi$)}: $\alpha = 0.5$.
\textit{Ours (no $G_\phi$, $\alpha{=}0$)}: deterministic coupling.
}
\label{fig:cifar_oracle_ablation}
\end{figure}

%% file: sec/supp/ablation.tex
\section{Ablation studies}
\label{sec:supp_ablations}

\subsection{Projection variants}
\label{sec:supp_projection_variants}

This section accompanies the claim
in \mcref{sec:ot_by_design}
that OT-optimality of the identity coupling
is necessary but not sufficient
for our prior design.
We compare four candidate transformations $T$,
each an orthogonal projection
that retains $1/16$ of the input dimensionality.
For all four,
the identity coupling $x_i \mapsto T(x_i)$
is OT-optimal by construction
(\cref{prop:ot_projection,prop:ot_discrete}),
yet only the low-frequency members
yield a usable prior in practice.

\myparagraph{Operator family.}
We instantiate four orthogonal projections
on CIFAR-10 images ($32 \times 32 \times 3$).
\emph{Bicubic} sets $T = \mathcal{U} \circ \mathcal{D}$
with $\mathcal{D}$ area downsampling $32 \times 32 \to 8 \times 8$
and $\mathcal{U}$ bicubic upsampling
(the operator used throughout the paper).
\emph{Fourier truncation} keeps the lowest-frequency $8 \times 8$ block
of DFT coefficients per channel
and zeros the rest.
\emph{Random pixel mask} applies a fixed binary mask,
sampled once and held constant across all images,
that retains a random subset of $1/16$ of the pixels.
\emph{Random patch mask} applies a fixed binary mask
over an $8 \times 8$ grid of $2 \times 2$ patches,
retaining $1/16$ of the patches.
A diagonal $0/1$ mask is symmetric and idempotent,
so both masking operators are valid orthogonal projections.
The first two preserve low-frequency structure.
The last two do not.

\myparagraph{Setup.}
For each operator we train $v_\theta$
at the full CIFAR-10 capacity (\cref{tab:cifar_arch})
for $150{,}000$ steps
under the noisy prior of \mcref{eq:noisy_prior},
once at $\alpha = 0$
and once at $\alpha = 0.5$.
We evaluate in the oracle condition
of \cref{sec:supp_oracle},
i.e., starting points are drawn
from real test images projected through $T$
rather than from the prior generator $G_\phi$.
This isolates the effect of the projection operator
on $v_\theta$
from the quality of the lightweight prior model.
Reconstruction MSE
$\mathbb{E}\,\|T(x) - x\|^2$
on the training set
is the squared transport cost
paid by the identity coupling
under the OT objective of \mcref{eq:ot}.
We additionally verify OT-identity preservation
empirically by computing the exact discrete OT assignment
on $10{,}000$ pairs (Hungarian algorithm)
for each operator at both noise levels.
Preservation is $100\%$ in all eight settings,
confirming that all four operators
are OT-equivalent under our objective,
both deterministically and under the $\alpha = 0.5$ noise interpolation.

\input{tables/supp_projection_variants}

\myparagraph{Results.}
\cref{tab:projection_variants} confirms
the necessary-but-not-sufficient claim.
Reconstruction MSE
splits the operators into two groups,
with the low-frequency family ($0.009$ and $0.010$)
roughly $30 \times$ lower
than the masking family ($0.271$ for both).
This reflects the spectral structure of natural images,
under which the lowest $8 \times 8$ frequency block
captures most of the $\ell_2$ energy,
while $1/16$ of the pixels
or $1/16$ of $2 \times 2$ patches
do not.

At $\alpha = 0$,
the low-frequency operators
yield curvature in the $10^{-6}$ to $10^{-5}$ range,
while the masking operators
sit two to three orders of magnitude higher
($\sim 1.4 \times 10^{-3}$).
FID at $\alpha = 0$
is dominated by the masking operators' high reconstruction error,
so $v_\theta$ is asked to in-paint
$15/16$ of the pixels
from a starting point that is largely zeros,
which is closer to the original generation problem
than to a residual-detail problem.

At $\alpha = 0.5$,
the regime we adopt,
the gap widens.
Bicubic and Fourier truncation
recover FID $\approx 88$,
while the masking operators
degrade to FID above $200$.
Two effects combine.
The transport cost paid by the masking operators,
though OT-optimal among permutations,
is intrinsically high
because the projection itself discards most of the signal.
And the prior support of the masking operators
is concentrated on pixel-level discontinuities
that are not stable under additive Gaussian noise,
so $v_\theta$ must traverse
a fragmented prior
on its way to the data manifold.

OT-optimality is therefore
necessary but not sufficient.
Hungarian preservation is $100\%$
for every operator and every noise level we test,
so the discriminating factor between operators
is not the OT property
but the tractability of the induced prior.
Within the low-frequency family,
bicubic edges out Fourier truncation
on reconstruction MSE,
matches it at $\alpha = 0.5$,
and is the operator we adopt.

\subsection{Downsampling factor}
\label{sec:supp_ablation_downsample}

The downsampling factor controls the dimensionality
of the low-frequency prior
and directly affects the trade-off
between prior informativeness and tractability.
Too little downsampling retains high-frequency content,
leaving $G_\phi$ with a generation task
comparable in difficulty to the original problem
and negating the benefit of a lightweight first stage.
Too much downsampling discards structure
needed for the flow model to produce high-quality samples.
\cref{fig:ds_factor_ablation_cifar} compares
$4\times$ and $8\times$ spatial downsampling on CIFAR-10.
A $4\times$ reduction
($8 \times 8$ prior from $32 \times 32$ images)
yields consistently better FID across step counts,
while $8\times$ downsampling degrades quality
as the prior loses too much structural detail
to serve as an effective starting point.
We adopt $4\times$ in all experiments.

\input{figs/ablation_downsample_factor/ablation_downsample_factor}

\subsection{\texorpdfstring{$G_\phi$}{G\_phi} step count}
\label{sec:supp_ablation_gen_nfe}

The number of ODE steps used by $G_\phi$
determines both the quality of the low-frequency prior
and the generation-time overhead.
Since our method targets the few-step regime,
we select the number of $G_\phi$ steps
to keep the prior-generation overhead
below $1$ effective NFE,
ensuring fair comparison
against single-step and two-step baselines.

\cref{fig:lf_steps_ablation} reports FID
as a function of effective NFE on CIFAR-10,
sweeping $G_\phi$ steps over $\{1, 2, 4, 8, 16\}$
and $v_\theta$ steps over $\{1, 2, 4, 8\}$.
At $1$ $v_\theta$ step,
using $16$ $G_\phi$ steps yields better FID
than $8$ steps,
but at a higher effective NFE ($\approx 2.4$ vs.\ $1.7$).
As the $v_\theta$ step count increases,
the advantage of additional $G_\phi$ steps diminishes.
At $2$+ $v_\theta$ steps,
$8$ $G_\phi$ steps achieve comparable or better FID
than $16$ at lower effective NFE.
We therefore use $8$ steps for $G_\phi$ in all experiments,
yielding an overhead of $\approx 0.72$ effective NFE
(see the Inference Overhead paragraph
in \cref{sec:supp_algorithms}).
Notably, even a single $G_\phi$ step
produces competitive results,
confirming that the low-frequency distribution
is sufficiently smooth
for coarse approximation at minimal cost.

\input{figs/small_generator_nfe_ablation/small_gen_ablation_fig}

%% file: tables/supp_projection_variants.tex
\begin{table}[t]
\centering
\footnotesize
\setlength{\tabcolsep}{5pt}
\renewcommand{\arraystretch}{1.15}
\caption[Projection variants on CIFAR-10.]{%
\textbf{Projection Variants on CIFAR-10.}
Four projections
(strictly orthogonal except for bicubic downsampling),
each retaining $1/16$ of the input dimensionality,
trained for $150{,}000$ steps
under \mcref{eq:noisy_prior}
at $\alpha = 0$ and $\alpha = 0.5$.
FID and curvature are evaluated at $1$ NFE
in the oracle condition
(real test images projected through $T$, $G_\phi$ bypassed).
Reconstruction MSE is $\mathbb{E}\,\|T(x) - x\|^2$
on the training set.
Hungarian assignment confirms
$100\%$ OT-identity preservation
for every operator at both noise levels.
The shaded row is the operator used throughout the paper.
Best per column in bold.}
\label{tab:projection_variants}

\begin{tabular}{lccccc}
\toprule
& & \multicolumn{2}{c}{$\alpha = 0$} & \multicolumn{2}{c}{$\alpha = 0.5$} \\
\cmidrule(lr){3-4} \cmidrule(lr){5-6}
Operator & MSE ($\downarrow$) & FID ($\downarrow$) & Curv.\ ($\downarrow$) & FID ($\downarrow$) & Curv.\ ($\downarrow$) \\
\midrule
\rowcolor{gray!10}
Bicubic              & $\mathbf{0.009}$ & $\mathbf{102.05}$ & $\mathbf{5.0{\times}10^{-6}}$ & $\mathbf{87.82}$ & $\mathbf{0.013}$ \\
Fourier truncation   & $0.010$          & $119.29$          & $1.2{\times}10^{-5}$          & $87.99$          & $\mathbf{0.013}$ \\
Random pixel mask    & $0.271$          & $117.90$          & $1.5{\times}10^{-3}$          & $226.72$         & $0.022$ \\
Random patch mask    & $0.271$          & $144.95$          & $1.4{\times}10^{-3}$          & $205.57$         & $0.023$ \\
\bottomrule
\end{tabular}
\end{table}

%% file: figs/ablation_downsample_factor/ablation_downsample_factor.tex
\begin{figure}[ht]
\centering
\includegraphics[width=0.5\linewidth]{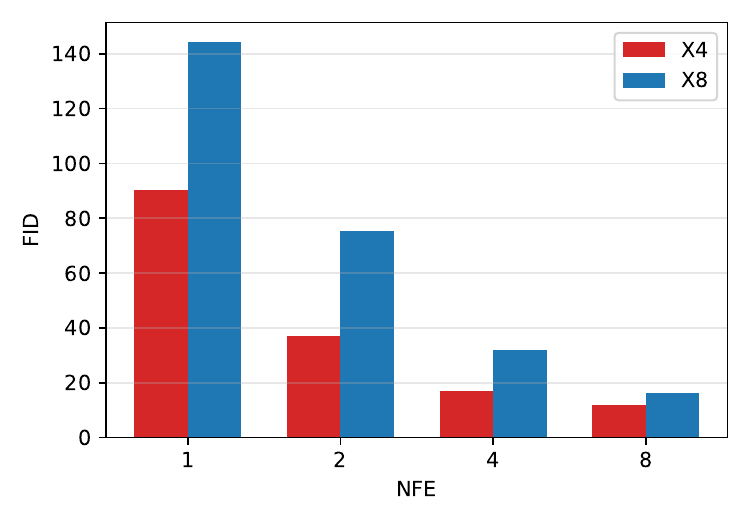}
\caption{\textbf{Downsample factor ablation on CIFAR-10.}
FID as a function of NFE for two spatial downsample factors
(\textcolor{red}{$\times 4$}) and (\textcolor{blue}{$\times 8$}).
A lower downsample factor consistently yields better FID across all NFE budgets}
\label{fig:ds_factor_ablation_cifar}
\end{figure}

%% file: figs/small_generator_nfe_ablation/small_gen_ablation_fig.tex
\begin{figure}[ht]
\centering
\includegraphics[width=0.5\linewidth]{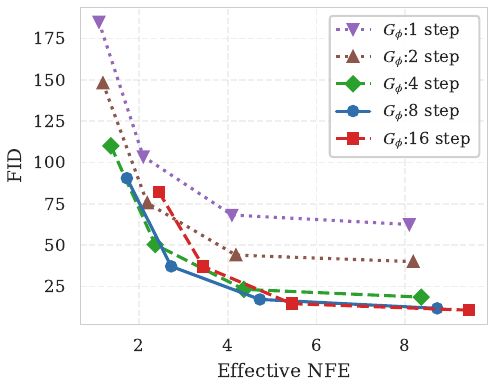}
\caption{\textbf{$G_\phi$ Step Count Ablation on CIFAR-10.}
Each curve fixes the number of $G_\phi$ steps (legend)
and sweeps $v_\theta$ over $\{1, 2, 4, 8\}$ steps.
The x-axis reports effective NFE,
the total generation cost
combining both $G_\phi$ and $v_\theta$ steps
($G_\phi$ steps weighted at $0.09\times$).
At $1$ $v_\theta$ step,
$16$ $G_\phi$ steps improve FID
over $8$ at higher effective NFE.
From $2$ $v_\theta$ steps onward,
$8$ $G_\phi$ steps achieve similar FID to $16$
at lower effective cost.}
\label{fig:lf_steps_ablation}
\end{figure}

%% file: sec/supp/proofs.tex
\section{Proofs and formal details}
\label{sec:supp_proofs}

\begin{proposition}[Orthogonal projections preserve self-proximity]
\label{prop:ot_projection}
Let $T\colon \mathbb{R}^d \to \mathbb{R}^d$ be an orthogonal projection,
\ie, $T^2 = T$ (idempotent) and $T^\top = T$ (symmetric).
Then for any $x, x' \in \mathbb{R}^d$,
\begin{equation}
    \label{eq:proj_nn}
    \|T(x') - x\|^2 = \|T(x) - x\|^2 + \|T(x') - T(x)\|^2.
\end{equation}
In particular, $\|T(x) - x\| \leq \|T(x') - x\|$,
with equality iff $T(x') = T(x)$.
\end{proposition}

\begin{proof}
Denote $\delta = x' - x$. By linearity of $T$,
\begin{equation}
    T(x') - x = T(x) + T(\delta) - x = T(\delta) + (T(x) - x).
\end{equation}
The first term $T(\delta) \in \operatorname{Im}(T)$.
For the second, note that by idempotency
\begin{equation}
T(T(x) - x) = T^2 (x) - T(x) = 0.
\end{equation}
so $(T(x) - x) \in \ker(T)$.
For an orthogonal projection,
$\operatorname{Im}(T)$ and $\ker(T)$ are orthogonal subspaces,
so by the Pythagorean theorem,
\begin{equation}
    \|T(x') - x\|^2
    = \|T(\delta)\|^2 + \|Tx - x\|^2
    = \|T(x') - T(x)\|^2 + \|T(x) - x\|^2.
\end{equation}
\hfill$\square$
\end{proof}

\begin{proposition}[Discrete OT optimality of projections]
\label{prop:ot_discrete}
Let $P\colon \mathbb{R}^d \to \mathbb{R}^d$ be an orthogonal projection,
and let $\{x_1, \ldots, x_n\} \subset \mathbb{R}^d$ be a finite point set
with distinct projections, \ie, $Px_i \neq Px_j$ for $i \neq j$.
Then the identity assignment $\sigma = \mathrm{id}$,
mapping $x_i \mapsto Px_i$,
minimizes the total transport cost
\begin{equation}
    \label{eq:transport_cost}
    C(\sigma) = \sum_{i=1}^{n} \|Px_{\sigma(i)} - x_i\|^2
\end{equation}
over all permutations $\sigma$ of $\{1, \ldots, n\}$.
\end{proposition}

\begin{proof}
By \cref{prop:ot_projection},
for any $i$ and any $j \neq i$,
\[
    \|Px_j - x_i\|^2 = \|Px_i - x_i\|^2 + \|Px_j - Px_i\|^2.
\]
Since the projections are distinct, $\|Px_j - Px_i\|^2 > 0$,
so
\begin{equation}
 \|Px_i - x_i\|^2 < \|Px_j - x_i\|^2.
\end{equation}
Every point $x_i$ is therefore strictly closer
to its own projection than to the projection of any other point.

Now consider any permutation $\sigma \neq \mathrm{id}$.
There exists at least one index $i$ with $\sigma(i) \neq i$.
For that index,
$\|Px_{\sigma(i)} - x_i\|^2 > \| Px_i - x_i\|^2$.
Since
$\|Px_{\sigma(j)} - x_j\|^2 \geq \|Px_j - x_j\|^2$
for all $j$, with strict inequality for at least one,
we have $C(\sigma) > C(\mathrm{id})$.
\end{proof}

\begin{remark}[Low-pass filtering as orthogonal projection]
\label{rem:parseval}
The discrete Fourier transform (DFT) is a unitary transformation.
Let $F$ denote the DFT matrix and let $P_S$
be the diagonal matrix that retains
a subset $S$ of Fourier coefficients
and zeros out the rest.
$P_S$ is an orthogonal projection
in the frequency domain.
Since conjugation by a unitary preserves
orthogonal projections,
$P = F^{-1} P_S F$ is an orthogonal projection
in pixel space.
In particular, choosing $S$
to be the lowest-frequency coefficients
yields low-pass filtering,
so \cref{prop:ot_projection,prop:ot_discrete}
apply directly.

In our pipeline (\mcref{sec:pipeline}),
the transformation $T = \mathcal{U} \circ \mathcal{D}$
applies downsampling $\mathcal{D}$ followed by upsampling $\mathcal{U}$.
When both operations use nearest-neighbor interpolation,
$T$ is exactly an orthogonal projection
onto the subspace of block-constant images,
and the OT guarantee holds by construction.
In practice, we find it better to use
area interpolation for $\mathcal{D}$
and bicubic interpolation for $\mathcal{U}$,
which is not strictly orthogonal
but yields lower reconstruction error,
and the OT structure is preserved empirically
as confirmed in \mcref{fig:ot:noise_vs_ot_swaps}.
\end{remark}

%% file: sec/supp/qualitative.tex
\section{Qualitative results}
\label{sec:supp_visualizations}

We visualize how the noise level $\alpha$
in \mcref{eq:noisy_prior}
affects the structure of the prior $x_0$,
in pixel space (CIFAR-10, \cref{sec:supp_vis_cifar})
and latent space (FFHQ, \cref{sec:supp_vis_ffhq}).
As $\alpha$ increases,
the low-frequency content that underlies the OT coupling
is progressively obscured,
consistent with the quantitative analysis
in \mcref{fig:ot:noise_vs_ot_swaps}
and \cref{sec:supp_pairwise}.

\subsection{Pixel-space prior (CIFAR-10)}
\label{sec:supp_vis_cifar}

\cref{fig:low2high_samples_cifar} shows representative CIFAR-10 samples
across all noise levels.
Each column displays the original image $x_1$,
its low-frequency projection $T(x_1)$,
and the noised starting point $x_0$
for $\alpha\in \{0, 0.1, \ldots 1\}$.
At low noise levels,
the dominant color, spatial layout,
and coarse object structure of the original image
remain clearly visible,
consistent with the quantitative OT preservation
reported in \mcref{fig:ot:noise_vs_ot_swaps}.
Around $\alpha \approx 0.5$, noise begins to dominate
and the visual correspondence to the original gradually fades.

\input{figs/ot_images_table/cifar_table}

\subsection{Latent-space prior (FFHQ)}
\label{sec:supp_vis_ffhq}

\cref{fig:low2high_samples_ffhq} shows the same visualization
for FFHQ in the latent space
of a pretrained autoencoder~\mcite{rombach2022stable_diffusion}.
Each column displays the original image,
its encoded latent,
the downsampled latent $x_1^{\downarrow} = \mathcal{D}(x_1)$,
and $x_0$ at increasing $\alpha$ values.
Unlike pixel space,
the structural correspondence
is harder to assess visually in latent representations,
particularly at moderate noise levels.
This is precisely why the quantitative validation
in \mcref{fig:ot:noise_vs_ot_swaps} is important,
where the identity coupling
is preserved at $100\%$ up to $\alpha = 0.5$
even in latent space.

\input{figs/ot_images_table/ffhq_table}

%% file: figs/ot_images_table/cifar_table.tex
\begin{figure}[ht]
\centering
\setlength{\tabcolsep}{0.2pt}
\providecommand{\imgw}{}
\renewcommand{\imgw}{0.1\textwidth}
\providecommand{\rowlbl}[1]{}
\renewcommand{\rowlbl}[1]{\adjustbox{valign=M}{\rotatebox{90}{\tiny #1}}}
\providecommand{\rowlblbig}[1]{}
\renewcommand{\rowlblbig}[1]{\adjustbox{valign=M}{\rotatebox{90}{\scriptsize #1}}}
\providecommand{\img}[1]{}
\renewcommand{\img}[1]{\adjustbox{valign=M}{\includegraphics[width=\imgw]{#1}}}
\ifdefined\rowsep\else\newlength{\rowsep}\fi\setlength{\rowsep}{14.5pt}
\adjustbox{max width=\linewidth, max height=0.78\textheight}{%
\begin{tabular}{ccccccccc}
\rowlblbig{$x_1$} & \img{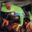} & \img{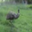} & \img{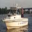} & \img{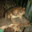} & \img{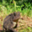} & \img{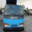} & \img{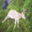} & \img{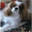} \\[\rowsep]
\rowlblbig{$x_1^{\downarrow}$} & \img{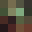} & \img{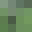} & \img{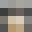} & \img{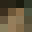} & \img{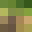} & \img{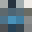} & \img{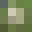} & \img{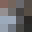} \\[1.5\rowsep]
\rowlbl{$\alpha{=}0$} & \img{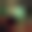} & \img{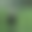} & \img{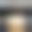} & \img{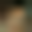} & \img{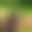} & \img{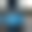} & \img{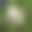} & \img{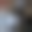} \\[\rowsep]
\rowlbl{$\alpha{=}.1$} & \img{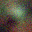} & \img{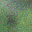} & \img{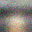} & \img{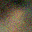} & \img{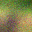} & \img{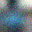} & \img{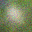} & \img{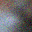} \\[\rowsep]
\rowlbl{$\alpha{=}.2$} & \img{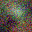} & \img{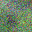} & \img{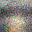} & \img{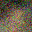} & \img{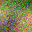} & \img{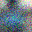} & \img{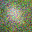} & \img{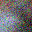} \\[\rowsep]
\rowlbl{$\alpha{=}.3$} & \img{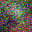} & \img{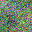} & \img{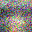} & \img{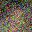} & \img{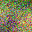} & \img{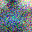} & \img{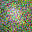} & \img{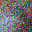} \\[\rowsep]
\rowlbl{$\alpha{=}.4$} & \img{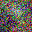} & \img{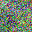} & \img{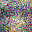} & \img{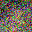} & \img{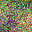} & \img{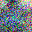} & \img{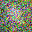} & \img{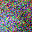} \\[\rowsep]
\rowlbl{$\alpha{=}.5$} & \img{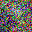} & \img{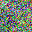} & \img{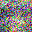} & \img{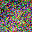} & \img{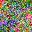} & \img{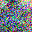} & \img{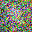} & \img{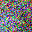} \\[\rowsep]
\rowlbl{$\alpha{=}.6$} & \img{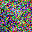} & \img{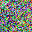} & \img{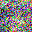} & \img{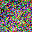} & \img{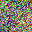} & \img{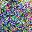} & \img{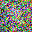} & \img{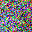} \\[\rowsep]
\rowlbl{$\alpha{=}.7$} & \img{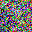} & \img{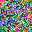} & \img{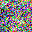} & \img{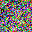} & \img{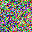} & \img{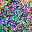} & \img{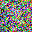} & \img{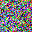} \\[\rowsep]
\rowlbl{$\alpha{=}.8$} & \img{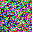} & \img{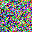} & \img{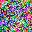} & \img{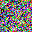} & \img{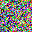} & \img{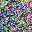} & \img{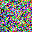} & \img{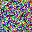} \\[\rowsep]
\rowlbl{$\alpha{=}.9$} & \img{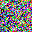} & \img{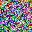} & \img{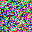} & \img{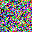} & \img{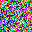} & \img{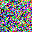} & \img{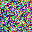} & \img{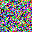} \\[\rowsep]
\rowlbl{$\alpha{=}1$} & \img{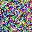} & \img{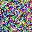} & \img{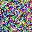} & \img{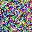} & \img{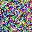} & \img{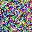} & \img{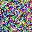} & \img{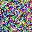} \\
\end{tabular}%
}
\caption{\textbf{Prior Construction on CIFAR-10.}
Each column shows one training image $x_1$ (top),
its low-frequency representation $x_1^{\downarrow} = \mathcal{D}(x_1)$,
and the noised starting point
$x_0 = (1 - \alpha)\,\mathcal{U}(x_1^{\downarrow}) + \alpha\,\epsilon$
(\protect\mcref{eq:noisy_prior})
at increasing noise levels,
ranging from the pure upsampled projection ($\alpha{=}0$)
through the operating point ($\alpha{=}0.5$),
where coarse structure remains visible,
to pure Gaussian noise ($\alpha{=}1$).}
\label{fig:low2high_samples_cifar}
\end{figure}

%% file: figs/ot_images_table/ffhq_table.tex
\begin{figure}[ht]
\centering
\setlength{\tabcolsep}{0.2pt}
\providecommand{\imgw}{}
\renewcommand{\imgw}{0.09\textwidth}
\providecommand{\rowlbl}[1]{}
\renewcommand{\rowlbl}[1]{\adjustbox{valign=M}{\rotatebox{90}{\tiny #1}}}
\providecommand{\rowlblbig}[1]{}
\renewcommand{\rowlblbig}[1]{\adjustbox{valign=M}{\rotatebox{90}{\tiny #1}}}
\providecommand{\img}[1]{}
\renewcommand{\img}[1]{\adjustbox{valign=M}{\includegraphics[width=\imgw]{#1}}}
\ifdefined\rowsep\else\newlength{\rowsep}\fi\setlength{\rowsep}{12.7pt}
\adjustbox{max width=\linewidth, max height=0.78\textheight}{%
\begin{tabular}{ccccccccc}
\rowlblbig{$I$} & \img{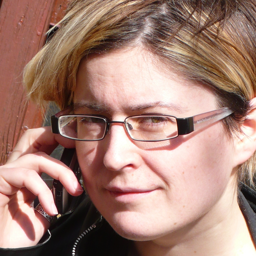} & \img{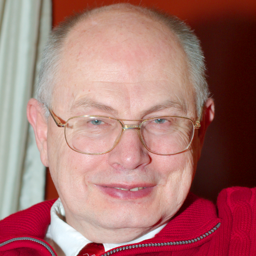} & \img{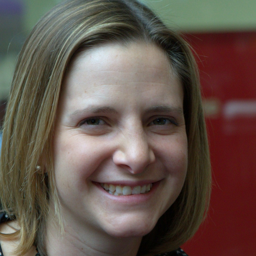} & \img{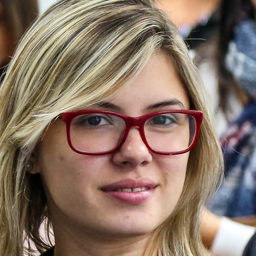} & \img{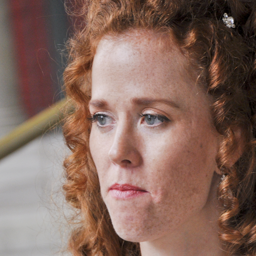} & \img{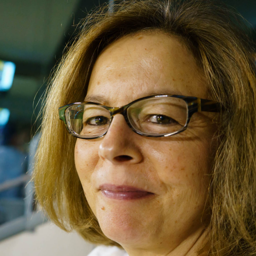} & \img{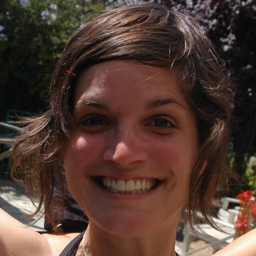} & \img{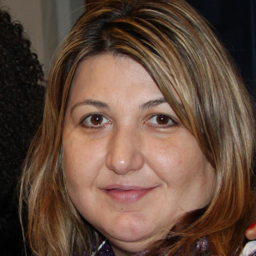} \\[\rowsep]
\rowlblbig{$x_1\!=\!\mathcal{E}(I)$} & \img{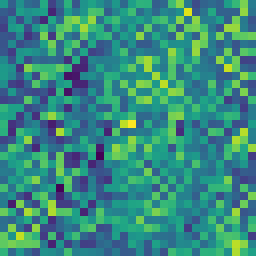} & \img{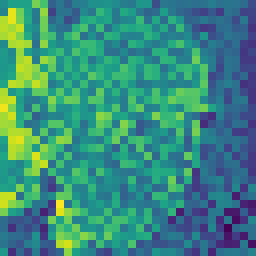} & \img{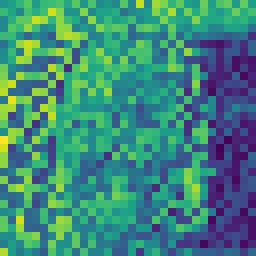} & \img{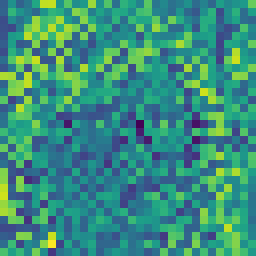} & \img{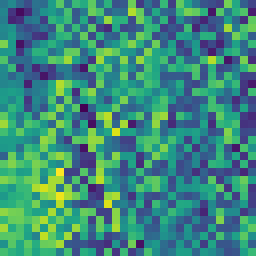} & \img{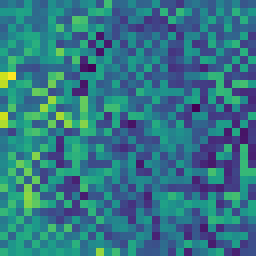} & \img{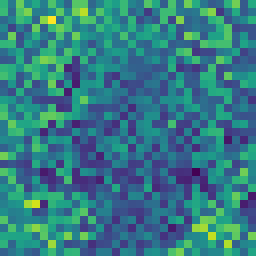} & \img{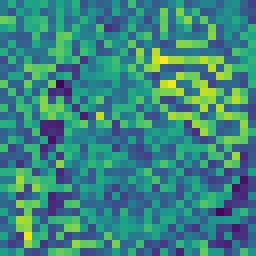} \\[\rowsep]
\rowlblbig{$x_1^{\downarrow}$} & \img{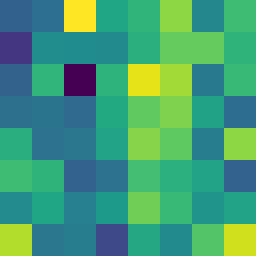} & \img{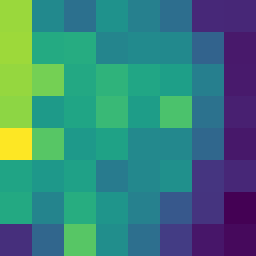} & \img{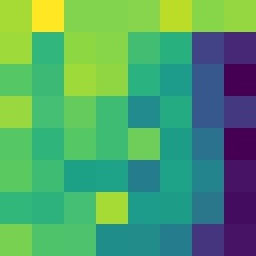} & \img{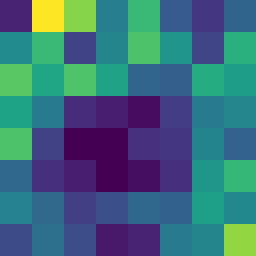} & \img{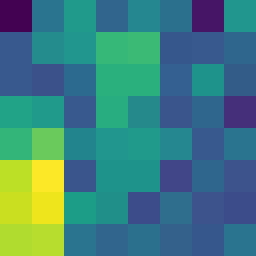} & \img{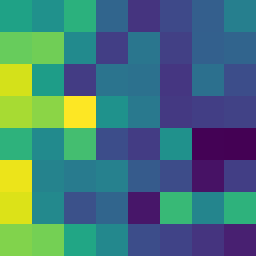} & \img{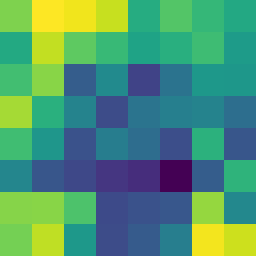} & \img{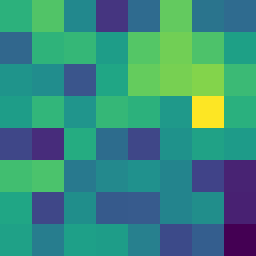} \\[1.5\rowsep]
\rowlbl{$\alpha{=}0$} & \img{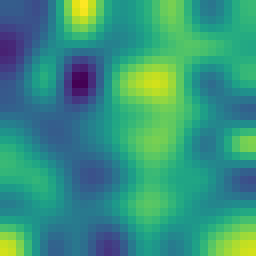} & \img{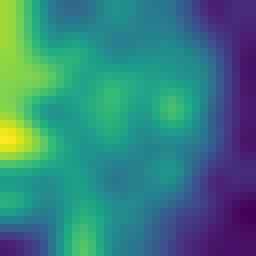} & \img{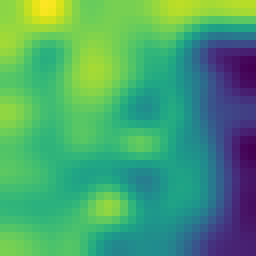} & \img{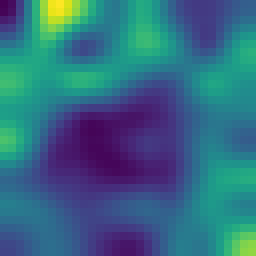} & \img{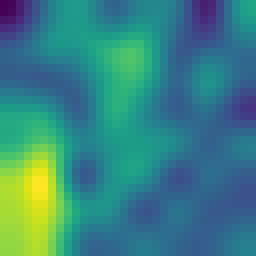} & \img{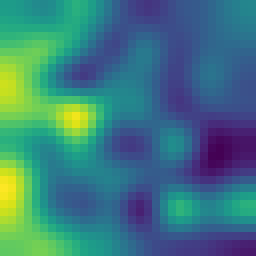} & \img{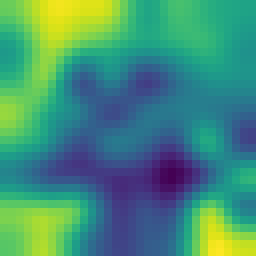} & \img{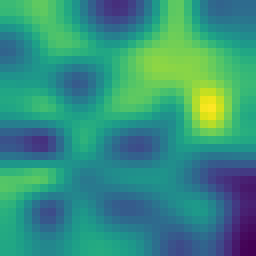} \\[\rowsep]
\rowlbl{$\alpha{=}.1$} & \img{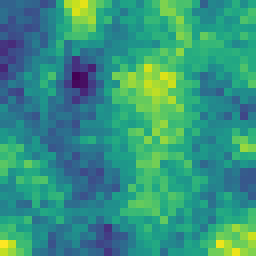} & \img{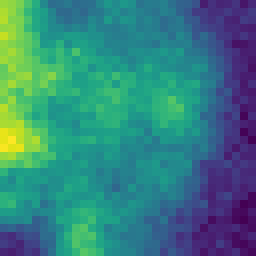} & \img{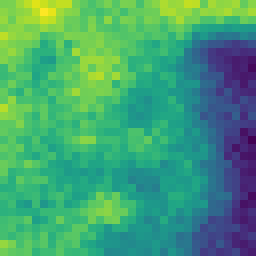} & \img{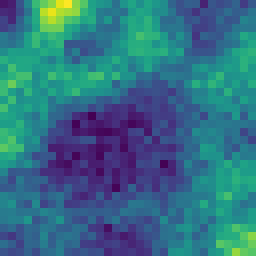} & \img{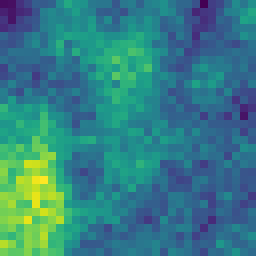} & \img{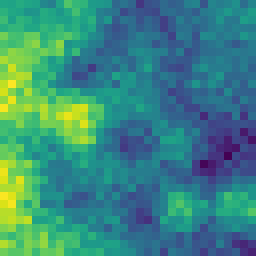} & \img{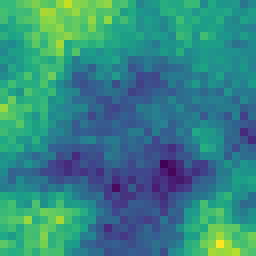} & \img{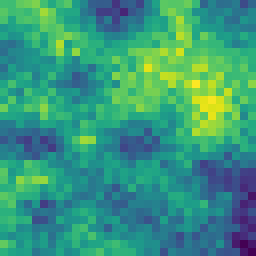} \\[\rowsep]
\rowlbl{$\alpha{=}.2$} & \img{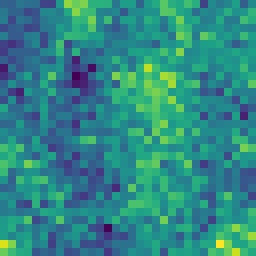} & \img{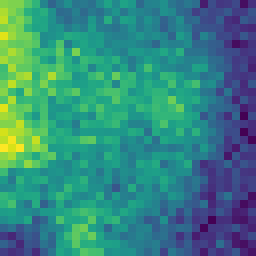} & \img{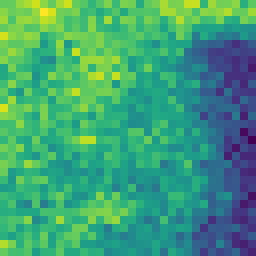} & \img{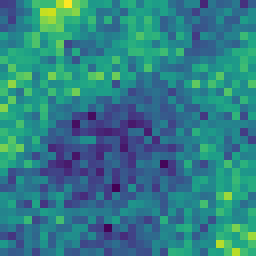} & \img{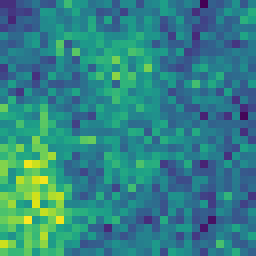} & \img{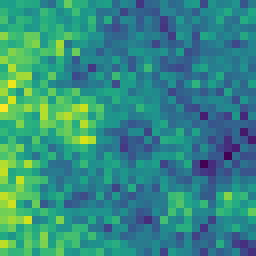} & \img{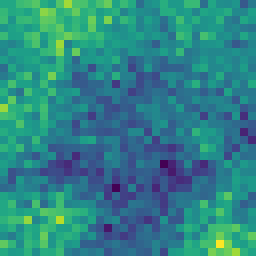} & \img{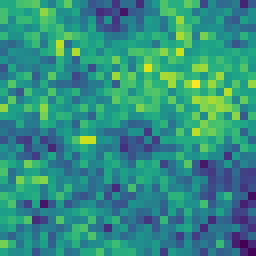} \\[\rowsep]
\rowlbl{$\alpha{=}.3$} & \img{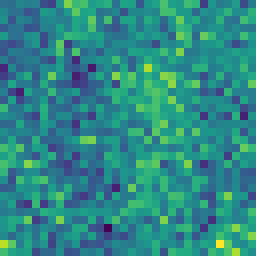} & \img{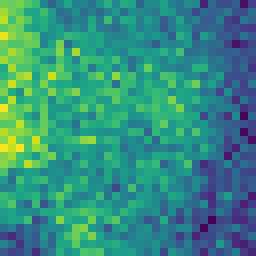} & \img{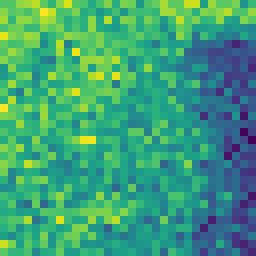} & \img{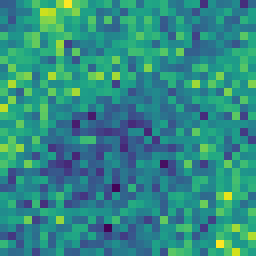} & \img{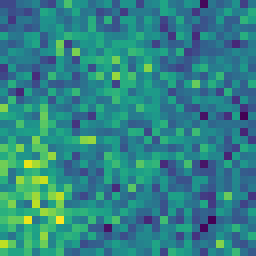} & \img{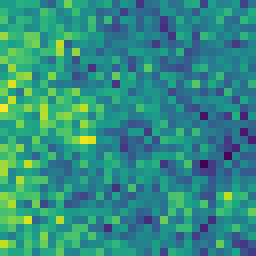} & \img{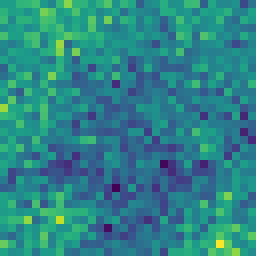} & \img{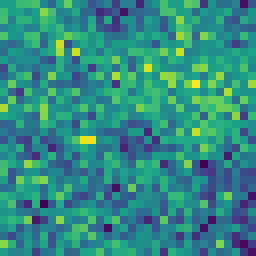} \\[\rowsep]
\rowlbl{$\alpha{=}.4$} & \img{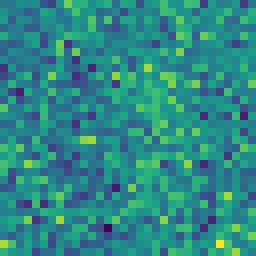} & \img{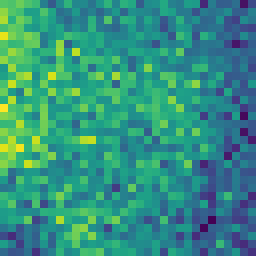} & \img{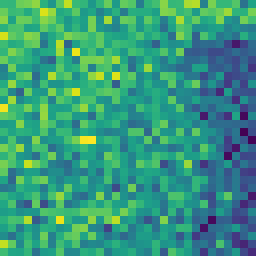} & \img{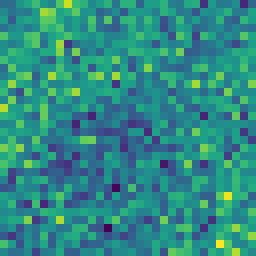} & \img{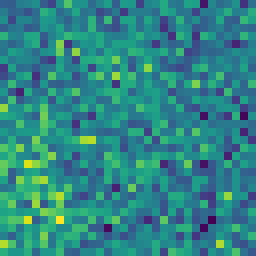} & \img{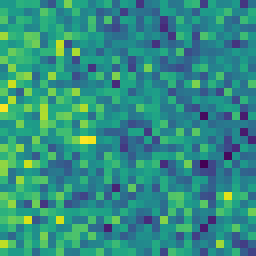} & \img{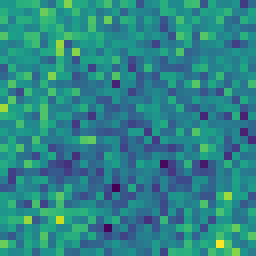} & \img{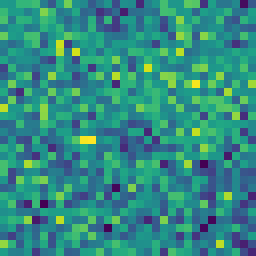} \\[\rowsep]
\rowlbl{$\alpha{=}.5$} & \img{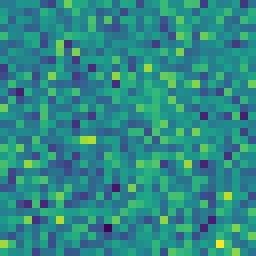} & \img{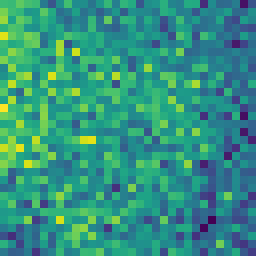} & \img{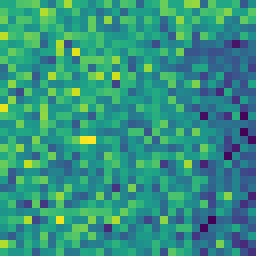} & \img{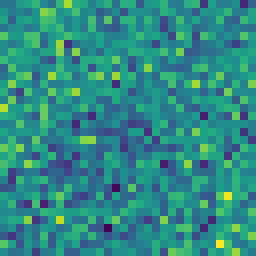} & \img{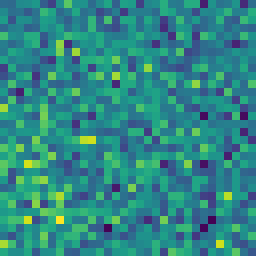} & \img{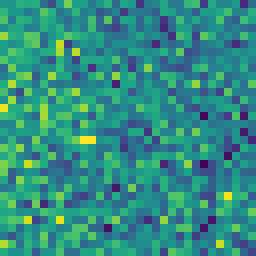} & \img{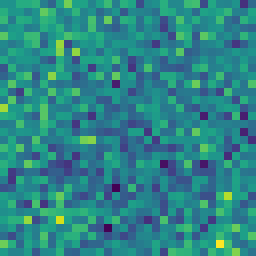} & \img{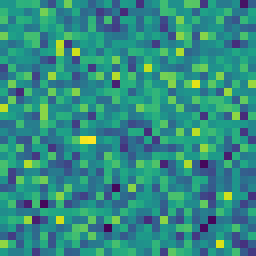} \\[\rowsep]
\rowlbl{$\alpha{=}.6$} & \img{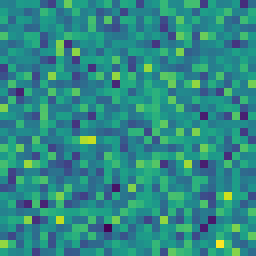} & \img{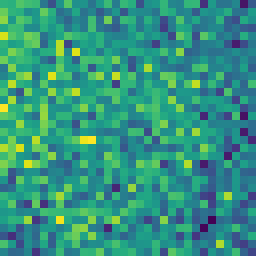} & \img{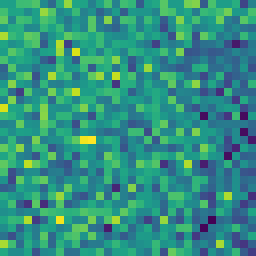} & \img{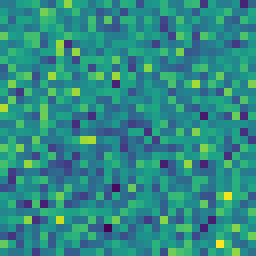} & \img{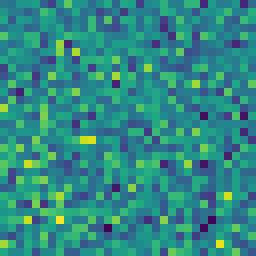} & \img{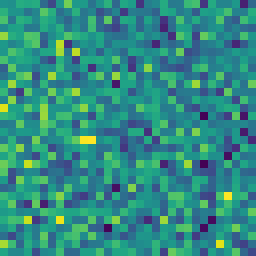} & \img{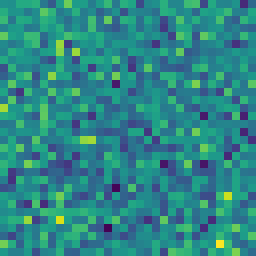} & \img{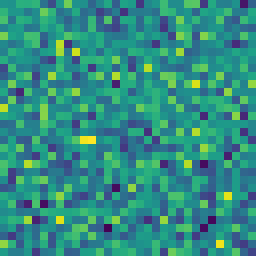} \\[\rowsep]
\rowlbl{$\alpha{=}.7$} & \img{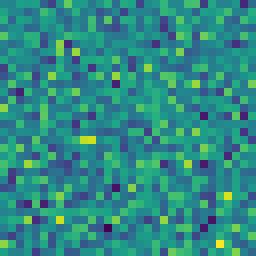} & \img{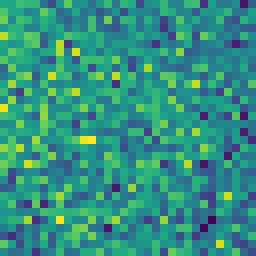} & \img{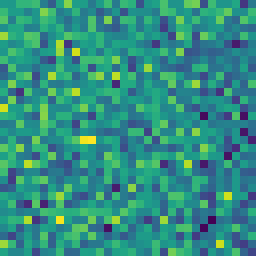} & \img{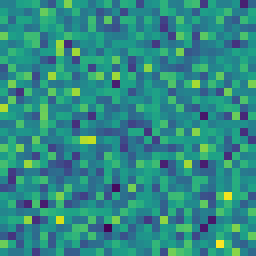} & \img{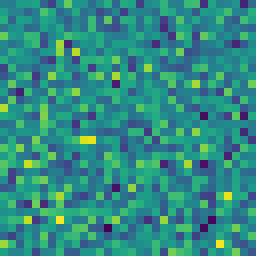} & \img{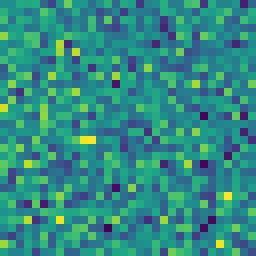} & \img{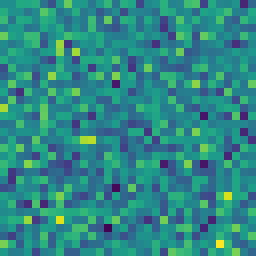} & \img{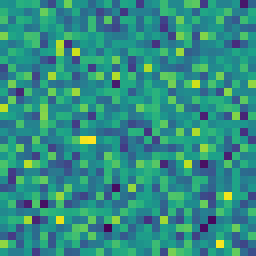} \\[\rowsep]
\rowlbl{$\alpha{=}.8$} & \img{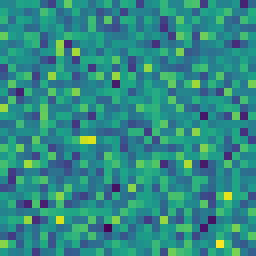} & \img{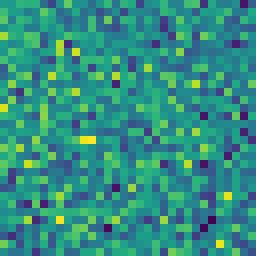} & \img{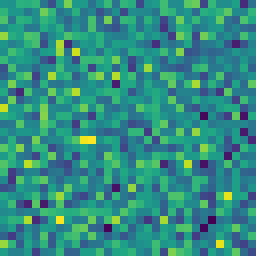} & \img{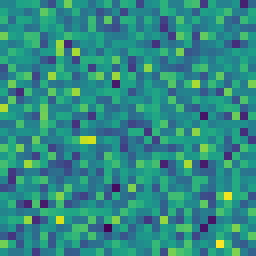} & \img{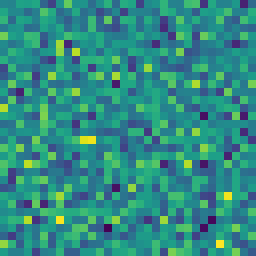} & \img{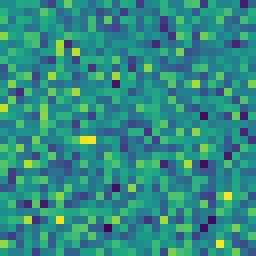} & \img{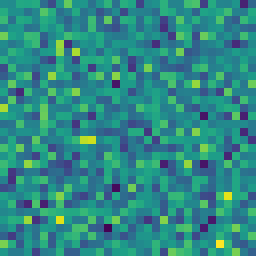} & \img{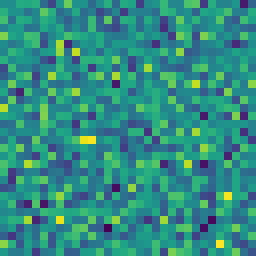} \\[\rowsep]
\rowlbl{$\alpha{=}.9$} & \img{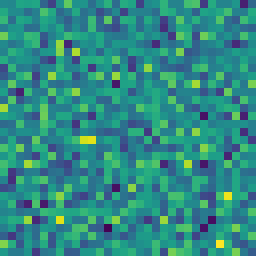} & \img{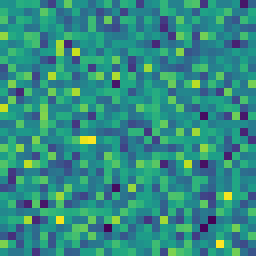} & \img{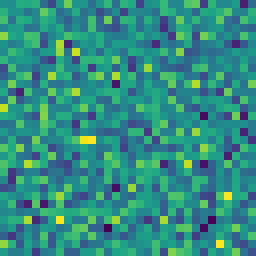} & \img{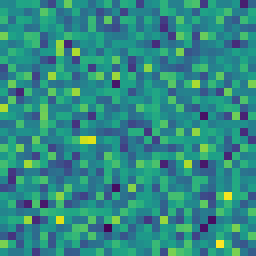} & \img{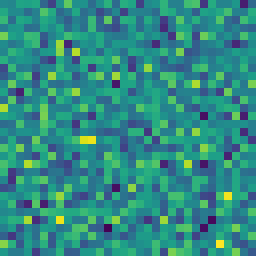} & \img{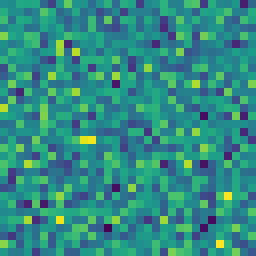} & \img{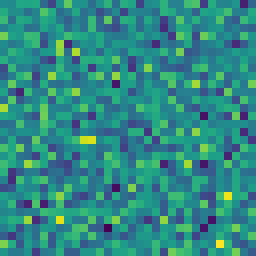} & \img{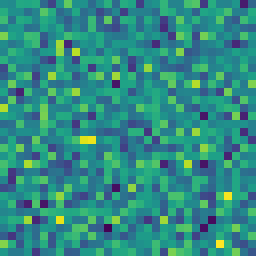} \\[\rowsep]
\rowlbl{$\alpha{=}1$} & \img{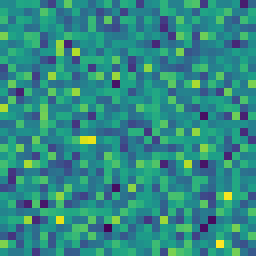} & \img{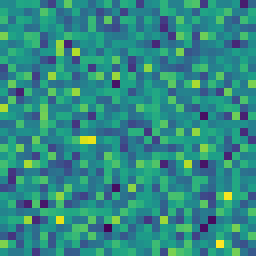} & \img{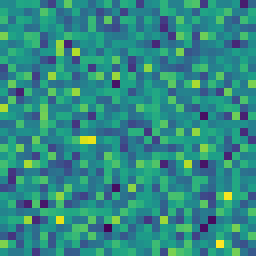} & \img{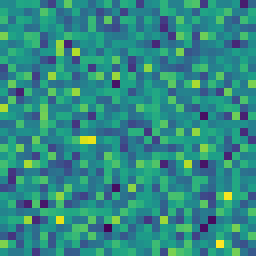} & \img{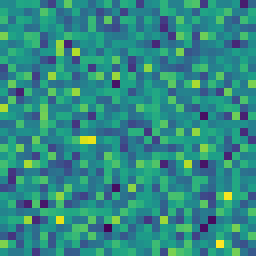} & \img{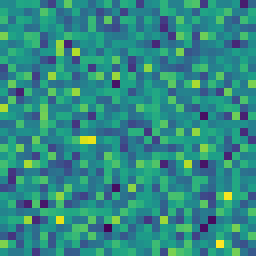} & \img{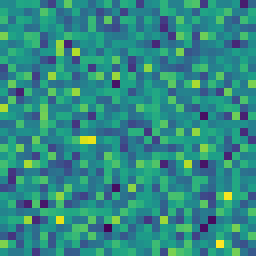} & \img{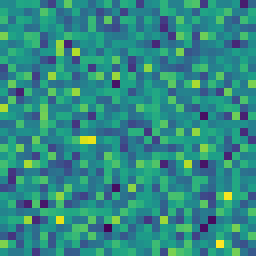} \\
\end{tabular}%
}
\caption{\textbf{Prior Construction on FFHQ (Latent Space).}
The first row shows original images $x_1$ for reference.
The second row visualizes the corresponding latent encoding,
and the third row its low-frequency representation $x_1^{\downarrow} = \mathcal{D}(x_1)$.
Subsequent rows show the noised starting point
$x_0 = (1 - \alpha)\,\mathcal{U}(x_1^{\downarrow}) + \alpha\,\epsilon$
(\protect\mcref{eq:noisy_prior})
at increasing noise levels,
ranging from the pure upsampled projection ($\alpha{=}0$)
through the operating point ($\alpha{=}0.5$)
to pure noise ($\alpha{=}1$).
All operations from the third row onward
are applied in the latent space
of a pretrained autoencoder~\protect\mcite{rombach2022stable_diffusion}.}
\label{fig:low2high_samples_ffhq}
\end{figure}